\documentclass[10pt,twocolumn,letterpaper]{article}

\usepackage{iccv}
\usepackage{times}
\usepackage{epsfig}
\usepackage{graphicx}
\usepackage{amsmath}
\usepackage{amssymb}

\usepackage{framed}
\usepackage[normalem]{ulem}
\usepackage{amsthm}
\usepackage{amsfonts}
\usepackage{enumerate}
\usepackage[utf8]{inputenc}
\usepackage[top=1 in,bottom=1in, left=1 in, right=1 in]{geometry}
\usepackage{graphicx}
\usepackage{float} 
\usepackage{subfigure}
\usepackage{booktabs}
\usepackage{color}
\usepackage[ruled,linesnumbered]{algorithm2e}
\usepackage{verbatim}
\usepackage{placeins}
\usepackage{wrapfig}
\usepackage{appendix}
\usepackage{bm}

\theoremstyle{definition}

\newtheorem{proposition}{Proposition}

\usepackage[pagebackref=true,breaklinks=true,letterpaper=true,colorlinks,bookmarks=false]{hyperref}

\iccvfinalcopy % *** Uncomment this line for the final submission

\def\iccvPaperID{12583} % *** Enter the ICCV Paper ID here

\ificcvfinal\fi

\begin{document}

%%%%%%%%% TITLE
\title{Elastic Interaction Energy-Based Generative Model: Approximation in Feature Space}

\author{Chuqi Chen$^{a}$\\
{\tt\small cchenck@connect.ust.hk}
\and
Yue Wu$^{a}$\\
{\tt\small  ywudb@connect.ust.hk}
\and
Yang Xiang$^{a,b}$\thanks{Corresponding author} \\
{\tt\small maxiang@ust.hk}\\
\small\textit{$^{a}$Department of Mathematics, Hong Kong University of Science and Technology,}\\ 
\small\textit{ Clear Water Bay, Kowloon, Hong Kong}\\
\small\textit{$^b$HKUST Shenzhen-Hong Kong Collaborative Innovation Research Institute,}\\
\small\textit{ Futian, Shenzhen, China}
}

\maketitle
\ificcvfinal\thispagestyle{empty}\fi

%%%%%%%%% ABSTRACT
\begin{abstract}
In this paper, we propose a novel approach to generative modeling using a loss function based on elastic interaction energy (EIE), which is inspired by the elastic interaction between defects in crystals. The utilization of the EIE-based metric presents several advantages, including its long range property that enables consideration of global information in the distribution. Moreover, its inclusion of a self-interaction term helps to prevent mode collapse and captures all modes of distribution. To overcome the difficulty of the relatively scattered distribution of high-dimensional data, we first map the data into a latent feature space and approximate the feature distribution instead of the data distribution. We adopt the GAN framework and replace the discriminator with a feature transformation network to map the data into a latent space. We also add a stabilizing term to the loss of the feature transformation network, which effectively addresses the issue of unstable training in GAN-based algorithms. Experimental results on popular datasets, such as MNIST, FashionMNIST, CIFAR-10, and CelebA, demonstrate that our EIEG GAN model can mitigate mode collapse, enhance stability, and improve model performance.
\end{abstract}

%%%%%%%%% BODY TEXT

\begin{figure*}[!hbtp]
\begin{center}
     \includegraphics[width=0.9\textwidth]{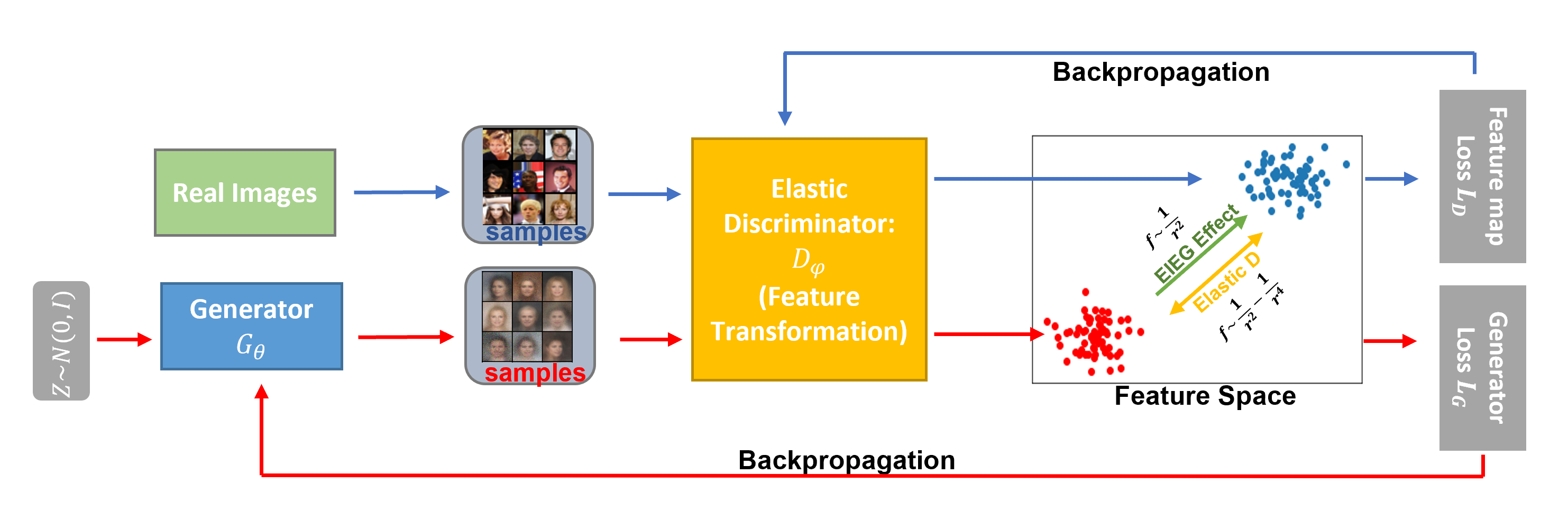}
\end{center}
\caption{Framework of EIEG GAN.}
\label{fig:framework}
\end{figure*}

\section{Introduction}

Deep generative models \cite{ref26,ref27,ref28} have gained significant attention in recent years due to their ability to generate new samples from the target distribution of the given complex dataset. 
There are three main approaches to generative modeling: likelihood-based methods \cite{ref29,ref30,ref31,ref32}, generative adversarial networks (GANs) \cite{ref1}, and score-based generative models \cite{ref8,ref9,ref10,ref11}.
In GAN-based algorithms, a critical point is to define a distance that appropriately measures the agreement
between the distribution of the generated samples $\mathbb{P}_{\theta}$ and the target distribution $\mathbb{P}_{data}$.
%
%In particular, a discriminator network is trained to minimize the predefined loss function so that
%$\mathbb{P}_{\theta}$ coincides with $\mathbb{P}_{data}$.
Different definitions of the distance between distributions lead to various GAN, e.g., WGAN \cite{ref3}, SobolevGAN \cite{ref7}, MMD-GAN \cite{ref4}, and others \cite{ref6,ref18,ref3,reffgan}. 
%Similarly, in score-based methods, different probabilistic metrics can determine the diffusion term of the stochastic differential equation (SDE) used to generate samples[]. %Thus, the choice of metric plays a crucial role in designing these generative modeling approaches.

This paper introduces a novel approach to generative modeling using a loss function based on elastic interaction energy (EIE). Specifically, the proposed elastic interaction-based loss function is inspired by the elastic interaction between defects in crystals \cite{refXIANG,refLUO}. This elastic interaction energy is used as the metric between $\mathbb{P}_{\theta}$ and $\mathbb{P}_{data}$ and $E[\mathbb{P}_{data},\mathbb{P}_{\theta}] = 0 \Leftrightarrow \mathbb{P}_{data} = \mathbb{P}_{\theta}$. This metric is utilized to derive the EIEG loss, which is endowed with a long-range property, capable of capturing global information about the distribution. Moreover, the inclusion of a self-interaction term within the EIEG loss fosters a repulsive force between samples, facilitating the generation of diverse samples and mitigating the issue of mode collapse. Empirical evaluations demonstrate the efficacy of the EIEG model in approximating smooth and compact data distributions in low dimensions.

%However, using a single neural network to transform is not ideal for high-dimensional datasets because the distribution of %high-dimensional data in high-dimensional space is relatively scattered. 
In practical problems, high-dimensional data often exhibits a more scattered distribution and complex structure.
To overcome this difficulty, we map the data into a low-dimensional latent feature space where data distribution is more compact. 
%Our training objective involves minimizing the EIEG Loss between the feature distributions of the real data and generated data, with the generator being optimized to minimize this loss. 
%Then instead of approximating the data distribution, we approximate the feature distribution. 
We generate the samples that minimize the EIEG loss of feature distributions instead of minimizing the distance in data distribution directly.
Specifically, we adopt the GAN framework by using a feature transformation network $D_{\phi}$ as a discriminator to map the data into a latent feature space. 
A generator network $G_{\theta}$ is then used to transform samples from a Gaussian distribution into the target high-dimension space.
%, which can be mapped into the feature distribution in the feature space.
The key idea of our method is demonstrated in Fig.\ref{fig:framework}.

GAN-based models often encounter the challenge of unstable training \cite{refSNGAN,refUnstable}. To address this issue, we add a stabilizing term to the loss of the feature transformation network $D_{\phi}$. We present experimental results to show the effectiveness of our proposed EIEG GAN on popular datasets such as MNIST \cite{ref15}, FashionMNIST \cite{ref20}, CIFAR-10 \cite{ref16}, and CelebA \cite{ref17}. Our experiment results demonstrate that the incorporation of the stabilizing term promotes stability in the training process, mitigates mode collapse, and yields improved model performance. Moreover, we observe that the feature space exhibits greater diversity when this term is included, thereby it enhancing the model's ability to capture the underlying data distribution. Experiments result show that EIEG GAN outperforms several standard GAN-based models\cite{ref1,ref18,ref3,ref2,ref4} in terms of both sample diversity and training stability.

The main contributions of this work are:
\begin{itemize}
    %new distance
    \item We propose a novel distribution metric based on elastic interaction energy (EIE) and use it to derive the EIEG loss. This metric enables us to capture global information about the distribution and effectively address the issue of mode collapse and also promote the generation of diverse samples.
    %GAN besed on new distance % feature discriminator
    \item Our EIEG GAN model adopts a feature transformation network as a discriminator to map high dimensional data into a lower dimensional latent feature space where the feature distribution is more compact. By doing so, instead of approximating the data distribution directly, we approximate the feature distribution.
    % extra term in $L_{D}$ 
    \item The introduction of a stabilizing term into the loss function of the feature transformation network represents a novel approach to addressing the issue of unstable training in GAN-based algorithms.
    % performance
    %\item Through experiments, we demonstrate that the EIEG GAN is competitive with basic GAN structures in image generation. The long range property of the elastic interaction loss enables our model to achieve more stability and faster convergence. 
    %We further examine the relationship between the feature space structure and generation quality under our framework.(this is part work of the paper but not contribution)
\end{itemize}
%Through experiments, we demonstrate that the EIEG GAN is competitive with basic GAN structures in image generation. The long range property of the elastic interaction loss enables our model to achieve more stability and faster convergence. We further examine the relationship between the feature space structure and generation quality under our framework.

\begin{comment}
The paper is organized as follows. In Section 2, we introduce a novel probability metric based on elastic interaction energy and derive the corresponding EIEG loss function. In Section 3, we propose the EIEG GAN method. Section 4 provides a comparison with related work. Section 5 presents the results of extensive experiments.
\end{comment}

\section{Elastic interaction energy-based generative model (EIEG)}

In this section, we propose a novel distribution metric based on Elastic Interaction Energy (EIE). We further adopt an EIEG Loss based on this metric. We demonstrate through experimental evaluations that under the guidance of EIEG Loss, only using a generator network can approximate low dimensional smooth data distribution well.

% \subsection{Elastic Interaction-Based Probability Metric}
% The elastic interaction-based probability metric is derived based on the elastic interaction between dislocations, which are line defects in crystals\cite{ref10}, and the lines move and interact to minimize the total energies. Inspired by this interaction energy, in the deep generative model, we consider the boundary of the area where the data samples are located and that of the generative samples as two curves with their own energies during the training process, the boundary of the generative samples is evolving under the supervision of their elastic interaction energy.

% Consider the elastic interaction energy of a single curve $\gamma$, which represents the boundary of a region. Here we show the elastic energy of the curve $\gamma$ in 2 dimensions is

% \begin{equation}
% E=\frac{1}{8 \pi} \int_\gamma \int_{\gamma^{\prime}} \frac{d \boldsymbol{l} \cdot d \boldsymbol{l}^{\prime}}{r}
% \end{equation}

% where $\gamma^{\prime}$ is $\gamma$ using parameters with prime, $d \boldsymbol{l}$ is the line element of $\gamma$ and 
%  $r=\sqrt{\left(x-x^{\prime}\right)^2+\left(y-y^{\prime}\right)^2}$ is the distance between these two points.

\subsection{Review of elastic interaction energy}

Our elastic interaction-based loss function is inspired by the elastic interaction energy between defects in crystals \cite{refXIANG,refLUO}. This interaction is characterized by its long-range effect. We consider a two-dimensional problem that the elastic interaction energy of the defects in the 2D space is
\begin{equation}
E=\sum_{i \in \mathbb{Z}} \sum_{j>i, j \in \mathbb{Z}}\left(\frac{\alpha_1}{r_{ij}^n} - \frac{\alpha_2}{r_{ij}^m} \right),
\end{equation}
where $r = \|\mathbf{x}_j-\mathbf{x}_i\|$, $\mathbf{x}_i$ stands for the position of the defect.
Specifically, for two defects fixed at locations $\mathbf{x}$ and $\mathbf{x}^{\prime}$, if the interaction energy between them is $\frac{1}{r^n}$, then there will be a repulsive force between them to minimize the energy. On the other hand, if the interaction energy between them is $-\frac{1}{r^n}$, then there will be an attractive force between them to minimize the energy.

%\subsection{Elastic interaction energy-based loss}
\subsection{EIEG metric}
Therefore, in deep generative model, we extend the elastic interaction energy to be the metric between two probability density functions. Consider two probability density function $p(\mathbf{x}), q(\mathbf{x}) : \mathbb{R}^{n} \rightarrow \mathbb{R}$
\begin{equation}
 \begin{aligned}
&E[p(\mathbf{x}),q(\mathbf{x}) ] \\&=\int_{\mathbb{R}^{n}}  \int_{\mathbb{R}^{n}}(p(\mathbf{x}) - q(\mathbf{x})) \cdot \frac{(p(\mathbf{y}) - q(\mathbf{y}))}{r^{n-1}} d \Omega_\mathbf{x} d \Omega_\mathbf{y},
 \end{aligned}
\end{equation}

% Due to the long-range nature, the minimization of the elastic interaction energy functional is not sensitive to the initialization. As a result, the proposed deep generative model with this elastic interaction-based loss function has demonstrated stable performance in the early stage of the training.

The proposed probability metric can be written in the following form, as shown in Eqn.(\ref{eq:EIEG_metric}):
\begin{equation}
\begin{aligned}
        E[p,q] &=\mathbb{E}_{x \sim p}\mathbb{E}_{y \sim p}(\frac{1}{r^{n-1}}) + \mathbb{E}_{x \sim q}\mathbb{E}_{y \sim q}(\frac{1}{r^{n-1}}) \\
        & \quad - 2\mathbb{E}_{x \sim p}\mathbb{E}_{y \sim q}(\frac{1}{r^{n-1}}).
\end{aligned}
\label{eq:EIEG_metric}
\end{equation}

In this expression, the first two terms in Eq. (\ref{eq:EIEG_metric}) correspond to the self-energy of the samples from their own distribution, while the last term represents the interaction energy between them.

\begin{proposition}

(i) Given two probability distribution $\mathbb{P}$ and $\mathbb{Q}$, we have $E[\mathbb{P},\mathbb{Q}] \geq 0$ and $E[\mathbb{P},\mathbb{Q}] = 0 \Leftrightarrow \mathbb{P} = \mathbb{Q}$.

 (ii) Let $\left\{\mathbb{P}_n\right\}$ be a sequence of distributions, and $\mathbb{P}_{n}$ is the corresponding probability density function. Considering $n \rightarrow \infty$, $E[\mathbb{P}_n,\mathbb{P}_{data}] \rightarrow 0 \Longleftrightarrow \|\mathbb{P}_n - \mathbb{P}_{data}\|_{\text{semi-}{H^{-\frac{1}{2}}}} \rightarrow 0$.

\end{proposition}
The details of the proof are provided in Appendix \ref{Appendix1}.

We propose to incorporate the EIE metric into the loss function of our generative model. It is important to note that the integrands of the double integrals in the long-range elastic energy-based metric contain singularities of the form $1/r^{n-1}$. A cut-off for the loss function should be established to address this challenge. Specifically, the distance between two samples, $\mathbf{x}\sim p$ and $\mathbf{y}\sim q$, both of which belong to $\mathbb{R}^n$, is set as follows:

\begin{equation}
e(x,y)= \begin{cases} \frac{1}{r^{n-1}}, & \text { if } r > R,
\\ 
(\frac{n+1}{n}R^{n} -\frac{1}{n}r^{n})\frac{1}{R^{2n-1}} , & \text { if } r \leq R, \end{cases}
\end{equation}
where $r = \|\mathbf{x} - \mathbf{y}\|$, $n$ is the dimension of the samples. The term for $r \leq R$ is designed to guarantee the smoothness of $e(x,y)$ and to make the gradient of $e(x,y)$ linear decay to zero when $x \rightarrow y$. Thus the objective function for the generative model we use is
\begin{equation}
\begin{aligned}
    \mathcal{M}_e(p,q) &= \mathbb{E}_{x \sim p}\mathbb{E}_{y \sim p}e(x,y) + \mathbb{E}_{x \sim q}\mathbb{E}_{y \sim q}e(x,y) \\
    & \quad - 2\mathbb{E}_{x \sim p}\mathbb{E}_{y \sim q}e(x,y).
\end{aligned}
\end{equation}

In practice, we use finite samples from distributions to estimate EIEG metric. Given $X = [{x_1,\dots,x_N}] \sim \mathbb{P}$ and $Y = [{y_1,\dots,y_N}] \sim \mathbb{Q}$, the estimator of $\mathcal{M}_e(p,q)$ is
\begin{equation}
\begin{aligned}
    \mathcal{\hat{M}}_e(X,Y)  &=  \frac{1}{N^2}\sum_{i,j=1}^{N}e(x_i,x_j) +  \frac{1}{N^2}\sum_{i,j=1}^{N}e(y_i,y_j)\\
    & \quad -  \frac{2}{N^2}\sum_{i,j=1}^{N}e(x_{i},y_{j}).
\end{aligned}
\end{equation}

When the proposed metric $\mathcal{M}_e(\mathbb{P},\mathbb{P}_n)$ is minimized with $\mathbb{P}$ fixed as the target distribution, the training process can be regarded as a physical evolution process to reduce the total energy. In this process, the generated samples $Y \sim \mathbb{P}_n$ are considered as particles that move towards the region of data samples $X \sim \mathbb{P}$ under the negative gradient of $\mathcal{\hat{M}}_e(X,Y)$, which can be regarded as force between $X$ and $Y$. The ultimate goal of this process is to make the distribution of the generated samples $\mathbb{P}_n$ approach the target distribution $\mathbb{P}$. By analyzing the negative gradient direction of $\mathcal{\hat{M}}_e(X,Y)$, we observe that the force between the generated samples and data samples is attractive. At the same time, the generated samples exhibit a repulsive force between each other, as shown in Fig.\ref{Fig:gradient}.
\begin{figure}[!hbtp]
    \centering
    \begin{subfigure}
        \centering
        \includegraphics[width=0.2\textwidth]{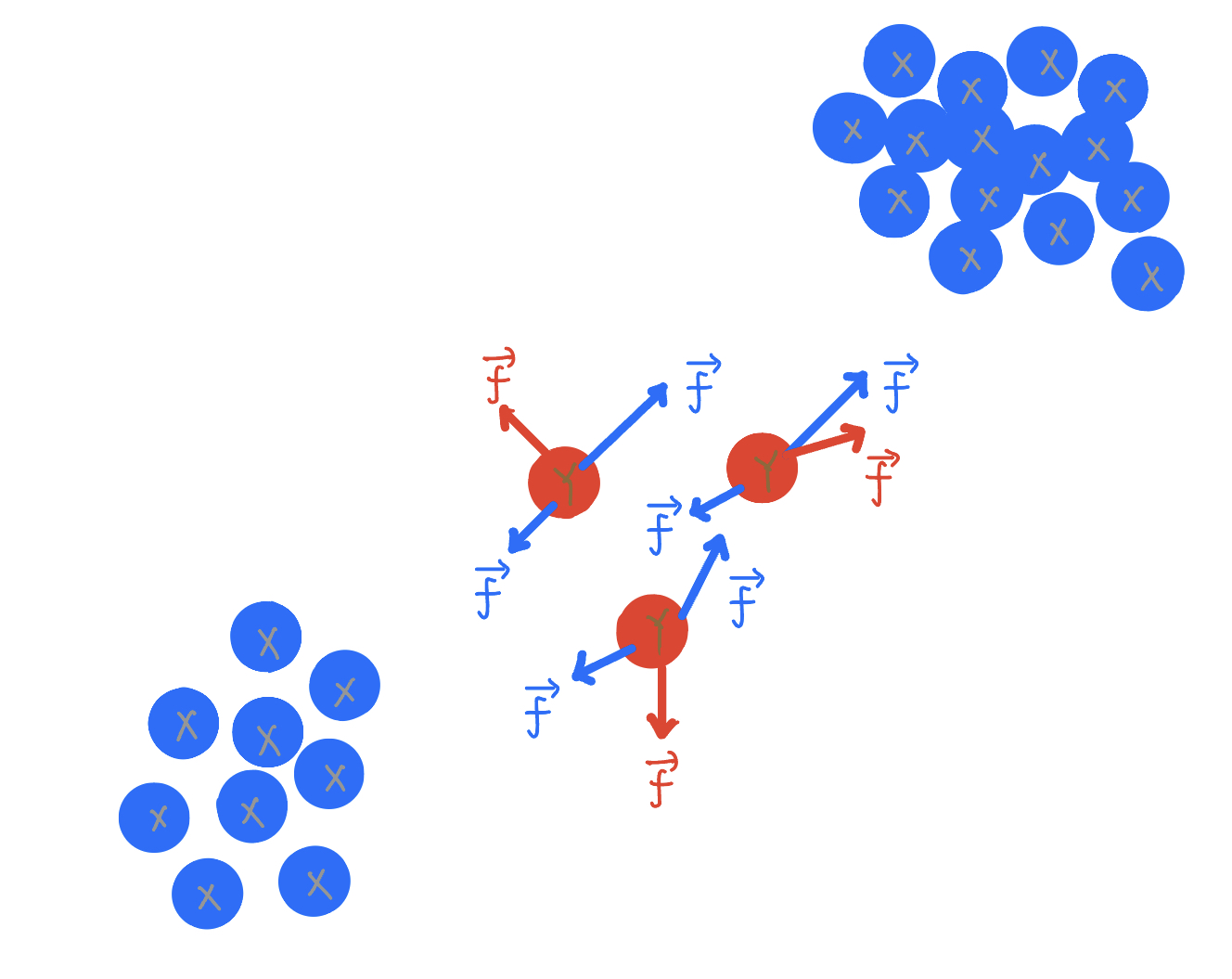}
        %\subcaption{}
    \end{subfigure}
    \begin{subfigure}
        \centering
        \includegraphics[width=0.2\textwidth]{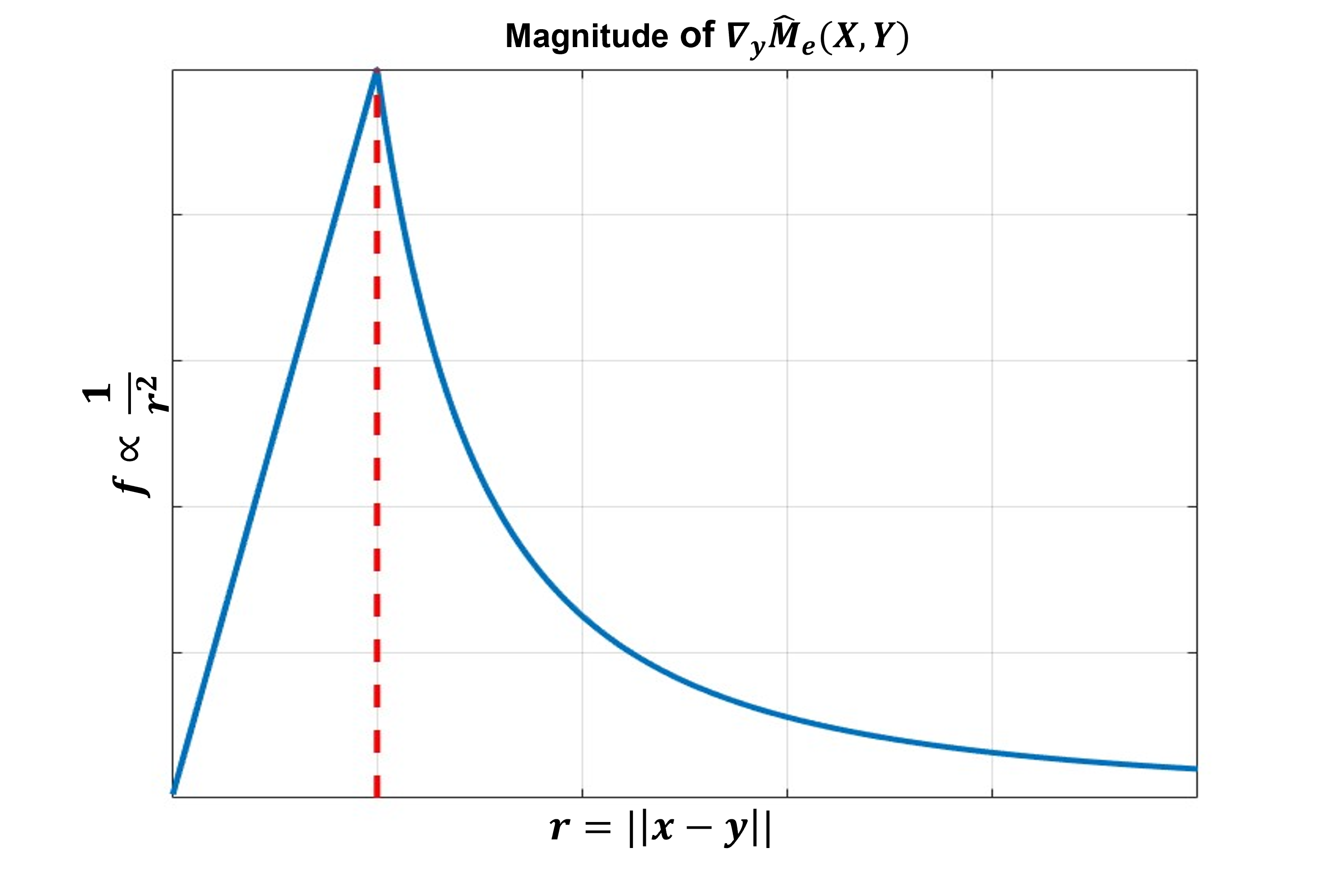}
        %\subcaption{}
    \end{subfigure}
    \caption{Left: gradient direction of $-\nabla_{Y_{j}} \mathcal{\hat{M}}_e(X,Y)$. Right: The magnitude of the force in 2D case. } 
\label{Fig:gradient}
\end{figure}

%\subsection{Elastic interaction energy-based generative model (EIEG) }
\subsection{EIEG loss to generative model}

We directly train a single generator network $G_{\theta}$ to map random variables  $Z \sim \mathcal{N}(0,I)$ into the samples of the target distribution under the guidance of our proposed loss, i.e.,we want $G_{\theta}(Z) \sim \mathbb{P}_{data}$. Thus the loss function for the generator network $G_{\theta}$
\begin{equation}
\begin{aligned}
    \min _\theta \mathcal{L} _{G} &=  \frac{1}{N^2}\sum_{i,j=1}^{N}e(G_{\theta}(Z_{i}),G_{\theta}(Z_{j})) \\
    & \quad -  \frac{2}{N^2}\sum_{i,j=1}^{N}e(X_{i},G_{\theta}(Z_{j})),
\end{aligned}
\label{Loss_gen_1}
\end{equation}
where $Z_{i} \sim \mathcal{N}(0,I)$, $X_{i} \sim \mathbb{P}_{data}$, and $N$ is the batch size during training. 

\paragraph{Experiments on smooth distributions}

We generate samples from a series of 2-dimensional Mixture of Gaussian (MoGs) \cite{ref8,ref13,ref14} by using a simple two layers fully connected network. The architecture and hyperparameters of the model are presented in Appendix \ref{Appendix4}. 

\begin{figure}[!hbtp]
\centering
    \begin{subfigure}%[t]{0.15\textwidth}
    \centering
    \includegraphics[width=0.15\textwidth]{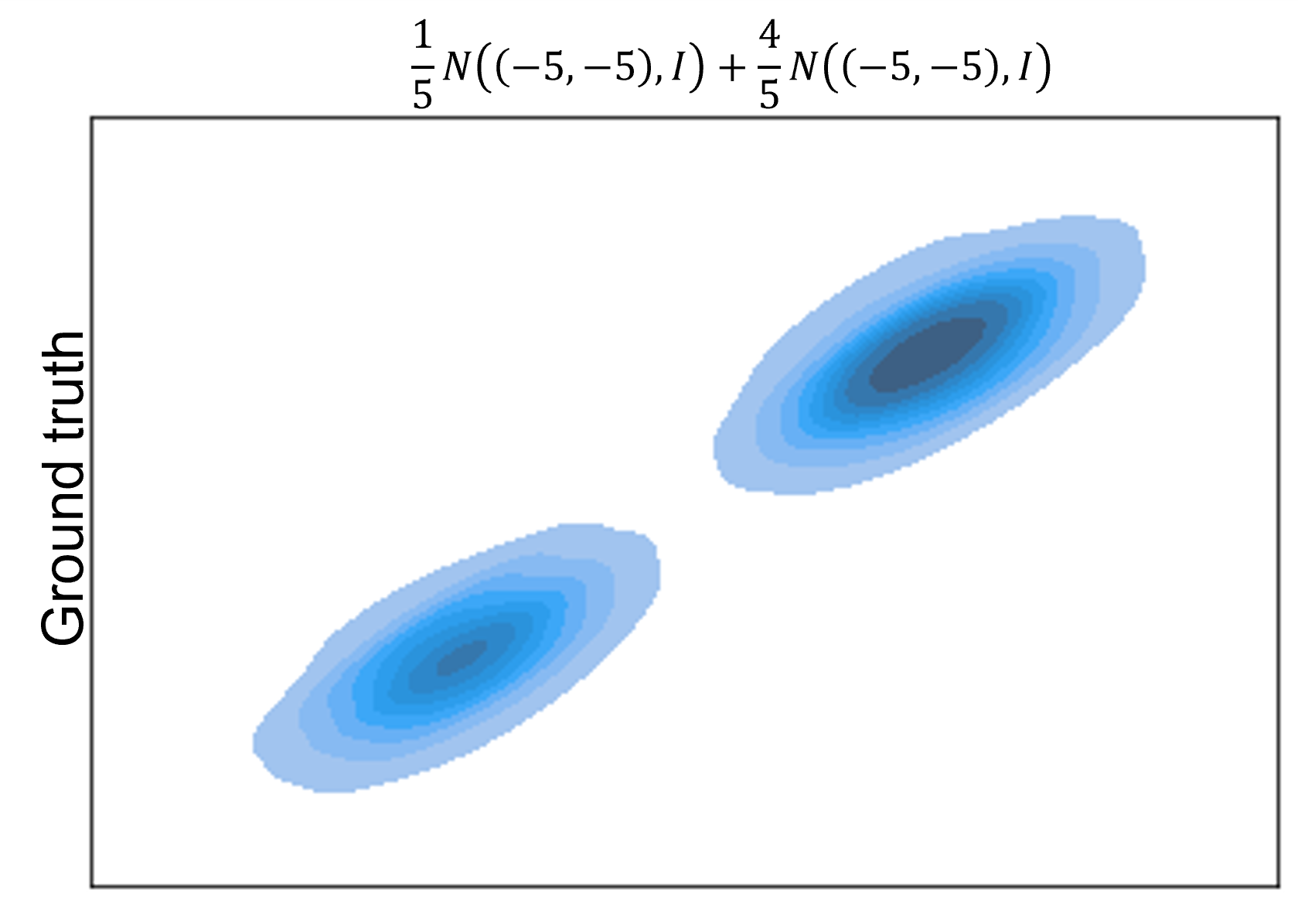}
    \label{Fig.sub.1}
    \end{subfigure}
    \begin{subfigure}%[t]{0.15\textwidth}
    \centering
    \includegraphics[width=0.15\textwidth]{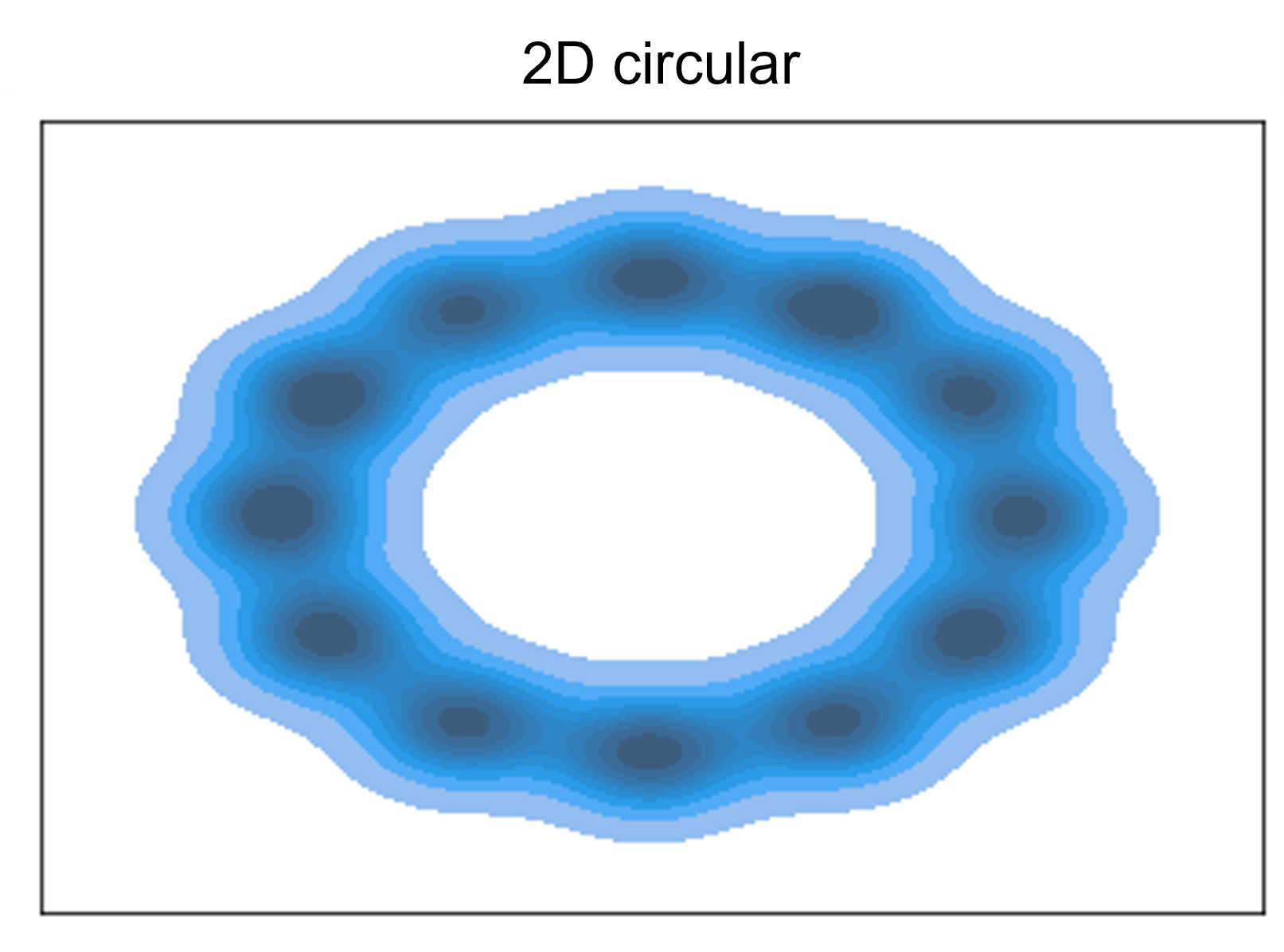}
    \label{Fig.sub.2}
    \end{subfigure}
    \begin{subfigure}%[t]{0.15\textwidth}
    \centering
    \includegraphics[width=0.15\textwidth]{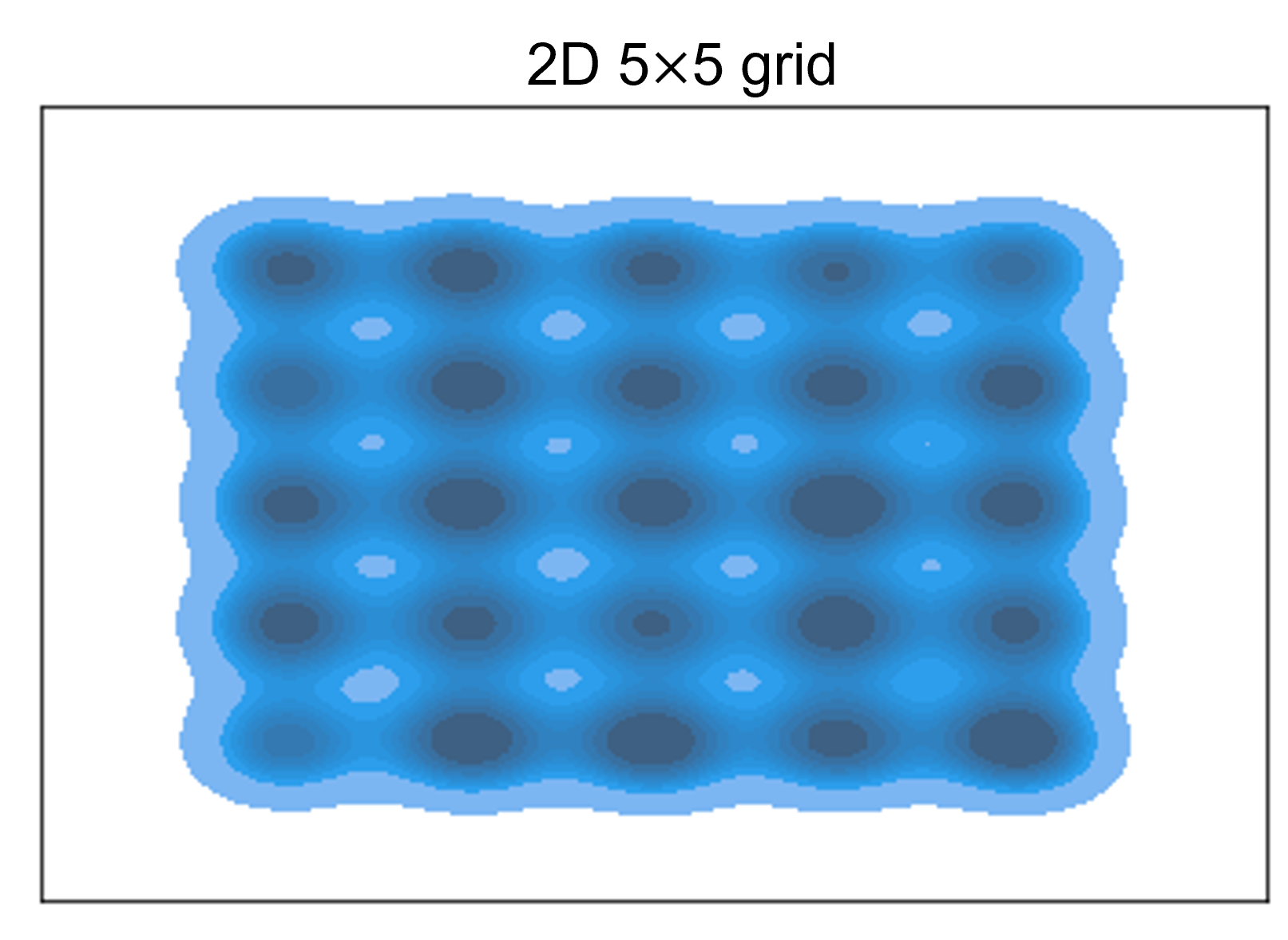}
    \label{Fig.sub.3}
    \end{subfigure}
    
\centering
    \begin{subfigure}%[t]{0.15\textwidth}
    \centering
    \includegraphics[width=0.15\textwidth]{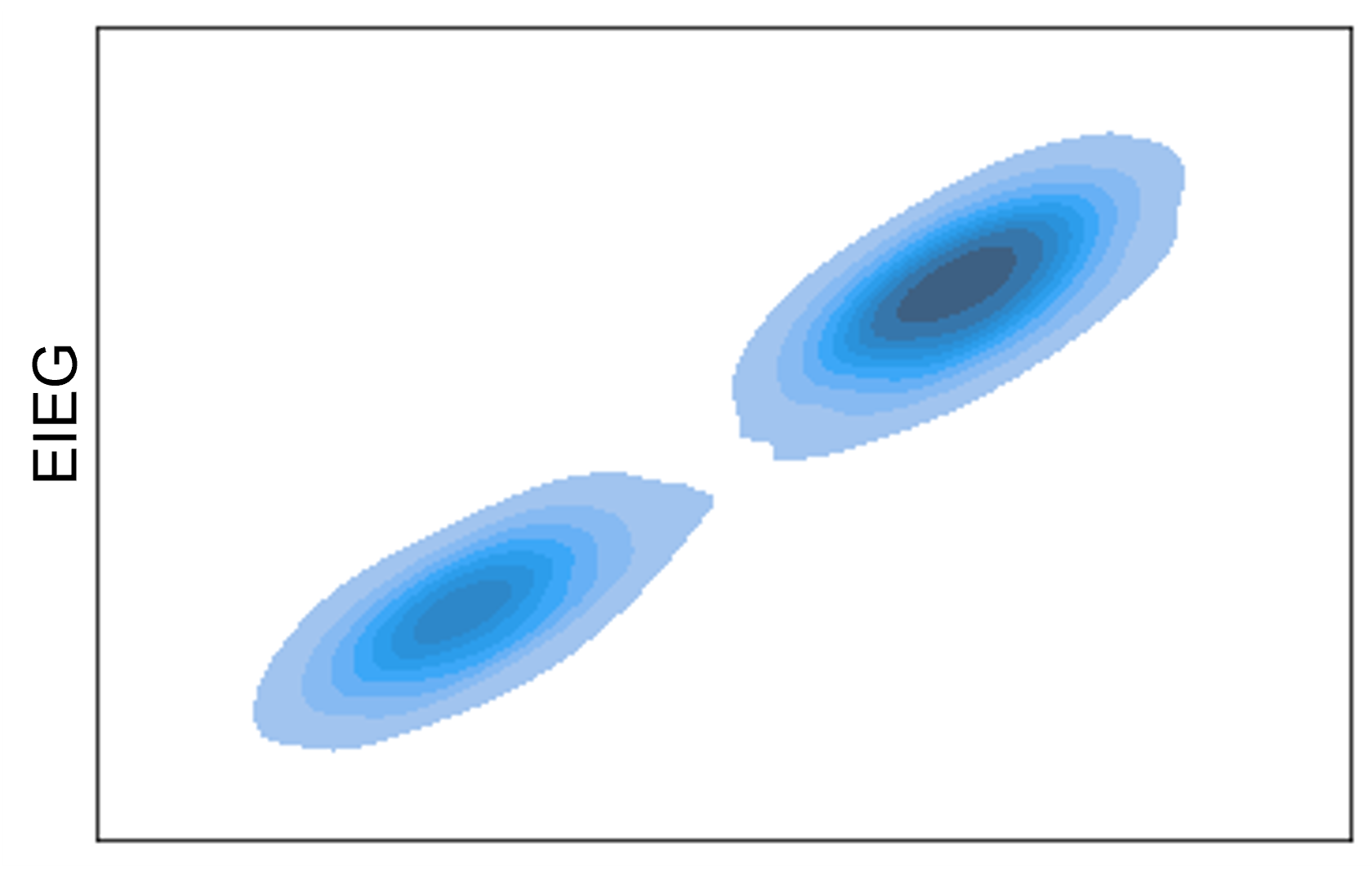}
    %\caption{}
    %\label{}
    \end{subfigure}
    \begin{subfigure}%[t]{0.15\textwidth}
    \centering
    \includegraphics[width=0.15\textwidth]{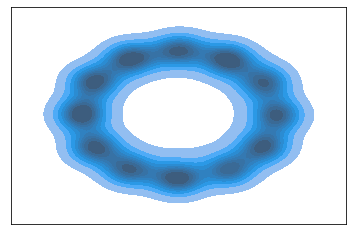}
    %\caption{}
    %\label{}
    \end{subfigure}
    \begin{subfigure}%[t]{0.15\textwidth}
    \centering
    \includegraphics[width=0.15\textwidth]{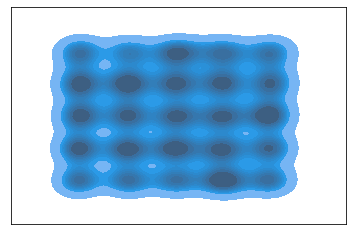}
    %\caption{}
    %\label{}
    \end{subfigure}
    \caption{First row: KDE plots for (a) Mixture of Gaussian with $p_{data} = \frac{1}{5} \mathcal{N}((-5,-5),I) + \frac{4}{5}\mathcal{N}((5,5),I)$ (b) 2-D mixture of eight Gaussians arranged in a circle (c) 2-D mixture  of 25 Gaussians arranged on $5\times5$ grid. Second row: distribution of samples generated by EIEG.}
\label{Fig:EIEG_Toy_Example}
\end{figure}

As shown in Fig.\ref{Fig:EIEG_Toy_Example}, under the guidance of the EIEG Loss  Eqn.\eqref{Loss_gen_1}, we find that simple neural networks can generate samples from data distributions and provide good approximations of the data distributions in low dimension.  

\paragraph{Importance of self-interation term}

In the EIEG loss Eqn.\eqref{Loss_gen_1}, the first term represents the self-interaction between the generated samples, reflecting the repulsive force among them. This term plays a significant role in the overall loss function and is crucial for ensuring the quality and diversity of the generated samples. Our experimental results demonstrate that eliminating this term causes the generative samples to collapse, as shown in Fig. \ref{Fig:collapse}. Therefore, we emphasize the importance of the repulsive force term in our proposed loss function for generating diverse samples and overcoming collapsing problems.

\begin{figure}[!hbtp]
    \centering
    \begin{subfigure}
        \centering
        \includegraphics[width=0.23\textwidth]{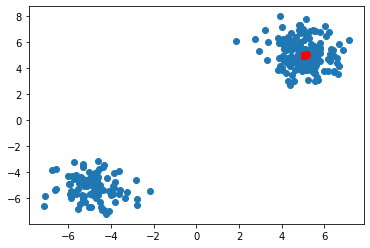}
    \end{subfigure}
    \begin{subfigure}
        \centering
        \includegraphics[width=0.23\textwidth]{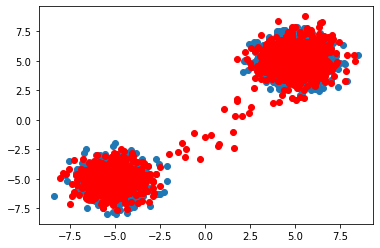}
    \end{subfigure}
    \caption{Data samples and generative samples for $ \frac{1}{5} \mathcal{N}((-5,-5),I) + \frac{4}{5}\mathcal{N}((5,5),I)$. Left: without self-interaction terms in $\mathcal{L}_{G}$. Right: with self-interaction terms. Blue points represent samples from the dataset, and red points represent the generated samples. }
\label{Fig:collapse}
\end{figure}

\FloatBarrier
%\subsubsection{Straightly approximate the data distribution}
\begin{comment}
\subsubsection{High dimensional examples}

We conduct experiments on high-dimension datasets MNIST[] and CIFAR-10[]. We find that only using a single generator network $G_{\theta}$ under the guidance of the EIEG Loss can not generate high-quality images from the high dimensional datasets. This suggests that the simple EIEG model is insufficient for generating quality samples from high-dimensional datasets with complex structures. The poor results can be attributed to the scattered distribution of data points in the high-dimensional space, which makes it difficult for our loss function to approximate. The proposed model effectively approximates situations where the sample distribution is denser but struggles when the distribution is more dispersed. (More discussions can be found in the supplementary material.)
\end{comment}

\begin{comment}
\begin{figure}[!hbtp]
    \centering
    \centering
    \begin{subfigure}
        \centering
        \includegraphics[width=0.23\textwidth]{figs/GEN_MNI_result_00.png}
        %\caption{MNIST}
    \end{subfigure}
    \begin{subfigure}
        \centering
        \includegraphics[width=0.23\textwidth]{figs/GEN_CIFAR_result_00.png}
        %\caption{CIFAR-10}
    \end{subfigure}
    \caption{Generated samples from EIEG with batch size 64.}
    %\label{}
\end{figure}
\end{comment}

%\FloatBarrier
\section{Elastic interaction energy-based GAN (EIEG GAN)}

In practical problems, we often have to deal with high-dimensional data with scattered distribution which makes it hard to directly approximate the data distribution. (More discussion can be found in Appendix \ref{Appendix2}.) To overcome this difficulty, we map the data into a low-dimensional latent feature space with compact distribution first, and then,  instead of approximating the data distribution, we are going to approximate the feature distribution. In this way, we can use the EIEG model to approximate the feature distribution easily.

\begin{comment}
Through the conducted experiments shown in Figure \ref{Fig:EIEG_Toy_Example}, our results demonstrate that our proposed loss function effectively guides a basic generator $G_{\phi}$ in approximating a distribution with a more condensed data distribution, specifically, the Mixture of Gaussian. Therefore, for different kinds of data distribution, we can first map the data into the feature space where the data distribution is denser in the feature space and approximates the feature distribution. The aim is to ensure that the data distribution in the feature space is more compact, which subsequently enhances the performance of the generative model based on our proposed loss function.
\end{comment}

\subsection{EIEG GAN}

To generate samples from high-dimensional data with complex distribution, we utilize two neural networks: a feature transformation network $D_{\phi}$ and a generator $G_{\theta}$. The former transforms the original data $X \sim \mathbb{P}_{data}$ into a feature space, while the latter maps random variables $Z \sim \mathcal{N}(0,I)$ into samples of the target distribution $\mathbb{P}_{data}$, guided by the EIEG loss through feature transformation network $D_{\phi}$. Our goal is to approximate the feature distribution such that $D{\phi}(G_{\theta}(Z)) \sim \mathbb{P}_{feature}$. 

EIEG GAN is under the framework of GAN \cite{ref1}. The method include a generator $G_{\phi}$ and a feature transformation network $D_{\phi}$ as 'discriminator'. We also add a stabilizing term in the loss of $D_{\phi}$. We call $D_{\phi}$ as elastic discriminator in our EIEG GAN. Also, the output size of the Elastic Discriminator is not 1, but rather the dimensionality of the feature space. This is because the Elastic Discriminator is not trying to discriminate between real and fake samples, but rather to transform the input samples into a feature space where it is easier for the generator to learn the underlying distribution of the data. Due to the extra stabilizing term, the loss of $G_{\theta}$ and the loss of $D_{\phi}$ are different. The training strategy we employ involves alternating between training the feature transformation network $D_{\phi}$ and the generator network $G_{\theta}$. Fig.\ref{fig:framework} illustrates our model's framework.

Specially, in Elastic Discriminator $D_{\phi}$ , the loss function is:
\begin{equation}
\begin{aligned}
    \max _\phi     \mathcal{L} _{D}&=\frac{1}{N^2}\sum_{i,j=1}^{N}\widetilde{e}_{\phi}(x_{i},x_{j}) + \frac{1}{N^2}\sum_{i,j=1}^{N}\widetilde{e}_{\phi}(y_{i},y_{j}) \\
    & \quad -  \frac{2}{N^2}\sum_{i,j=1}^{N}\widetilde{e}
    _{\phi}(x_{i},y_{j}),
\end{aligned}
\end{equation}
where $ {e}_{\phi}(x,y) = e(D_{\phi}(x),D_{\phi}(y)) - \varepsilon e_{s}(D_{\phi}(x),D_{\phi}(y)))$, $e_{s}(x,y)$ is the stabilizing term in Eq.\eqref{Eqn:Higher_order_e} in which $m>n$.

\begin{equation}
\begin{aligned}
    e_{s}(x,y) = \begin{cases} \frac{1}{r^{m-1}}, & \text { if } r>R \\ (\frac{m+1}{m}R^{m} - \frac{1}{m}r^{m})\frac{1}{R^{2m-1}}, & \text { if } r \leq R\end{cases}.
\end{aligned}
\label{Eqn:Higher_order_e}
\end{equation}

Specially, in Generator $G_{\phi}$, the loss function is :
\begin{equation}
\begin{aligned}
    \min _\theta \mathcal{L} _{G} &=  \frac{1}{N^2}\sum_{i,j=1}^{N}e_{\phi}(G_{\theta}(Z_{i}),G_{\theta}(Z_{j})) \\
    & \quad -  \frac{2}{N^2}\sum_{i,j=1}^{N}e_{\phi}(X_{i},G_{\theta}(Z_{j})),
\end{aligned}
\label{Loss_gen}
\end{equation}
where $Z_{i} \sim \mathcal{N}(0,I)$, $X_{i} \sim \mathbb{P}_{data}$, and $N$ is the batch size during training. The training strategy we employ involves alternating between training the feature transformation network $D_{\phi}$ and the generator network $G_{\theta}$. 

The proposed algorithm is shown in Algorithm \ref{Alg: EIEG-GAN}.

\begin{algorithm}\label{Alg: EIEG-GAN}
\caption{EIEG GAN}
\SetKwData{Left}{left}\SetKwData{This}{this}\SetKwData{Up}{up}
\SetKwFunction{Union}{Union}\SetKwFunction{FindCompress}{FindCompress}
\SetKwInOut{Input}{\textbf{input}}\SetKwInOut{Output}{\textbf{output}}
\Input{$\alpha$ the learning rate, $B$  the batch size, $Epoch$ the training epoch,$n_c$  the number of iterations of discriminator per generator update.}
initialize generator parameter $\theta$  and discriminator parameter  $\phi$\;
\For{$epoch\leftarrow 1$ \KwTo $Epoch$}{
    \For{$k\leftarrow 1$ \KwTo $n_c$}{\label{forins}
      Sample a minibatches $\left\{x_i\right\}_{i=1}^B \sim \mathbb{P}(\mathcal{X})$  and $\left\{z_j\right\}_{j=1}^B \sim \mathcal{N}(0,I)$\; 
      $D_\phi \leftarrow \nabla_\phi \mathcal{L}_{D}$\; 
      $\phi \leftarrow \phi+\alpha \cdot \operatorname{Adam}\left(\phi, D_\phi\right)$\;
    }
    Sample a minibatches $\left\{x_i\right\}_{i=1}^B \sim \mathbb{P}(\mathcal{X})$  and $\left\{z_j\right\}_{j=1}^B \sim \mathcal{N}(0,I)$\; 
    $G_\theta \leftarrow \nabla_\phi \mathcal{L}_{G}$\; 
      $\theta \leftarrow \theta-\alpha \cdot \operatorname{Adam}\left(\theta, G_\theta\right)$}
\Output {$G_{\theta}(z_{j})$ the generative samples.}
\end{algorithm}

It is worth noting that the loss function for $D_{\phi}$ is based on prior research \cite{refXIANG,refLUO}. The inclusion of the stabilizing term $e_{s}(x,y)$ in $\mathcal{L}_{D}$ greatly enhances the stability of the training process for our proposed algorithm. Furthermore, this stabilizing term serves as a repulsive force, preventing the collapse of data points in the feature space, which ultimately results in the generation of higher quality samples.

There are several advantages of EIEG GAN.

\paragraph{(1) Necessity of feature transformation network} 
We conduct experiments on utilizing a single generator network, $G_{\phi}$, in generating samples from two widely used datasets, MNIST and CIFAR-10, using our proposed loss function. Fig.\ref{Fig:EIEG} demonstrates the outcomes of our experiment. However, we observe poor results, which can be attributed to the scattered distribution of data points in the high-dimensional space, rendering the approximation of our loss function challenging. Therefore, we conclude that for datasets with intricate structures and high dimensionality, the EIEG-GAN architecture, which incorporates a feature transformation network, is necessary to attain optimal results.

\begin{figure}[!hbtp]
    \centering
    \centering
    \begin{subfigure}
        \centering
        \includegraphics[width=0.23\textwidth]{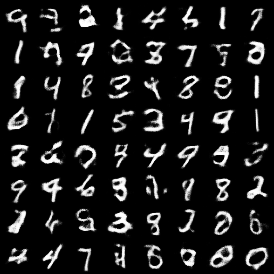}
        %\caption{MNIST}
    \end{subfigure}
    \begin{subfigure}
        \centering
        \includegraphics[width=0.23\textwidth]{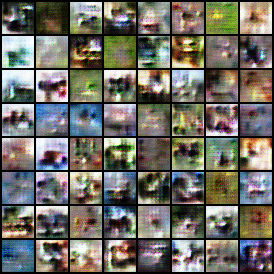}
        %\caption{CIFAR-10}
    \end{subfigure}
    \caption{Generated samples on MNIST and CIFAR10 using only a generator under our proposed loss.}
    \label{Fig:EIEG}
\end{figure}
\FloatBarrier
%\subsection{Loss function for feature transformation network}
%\FloatBarrier
%\subsection{Loss function for feature transformation network}

\paragraph{(2) Stable training process} 
One of the advantages of our EIEG GAN model lies in the incorporation of the stabilizing term in $\mathcal{L}_{D}$, which ensures a stable training process. The absence of the stabilizing term can render the training of EIEG GAN unstable. To illustrate this, we plot the learning curves of our models and evaluate the quality of our generative samples from the MNIST and CIFAR-10 datasets using Inception Scores, as shown in Fig.\ref{Fig:EIEG:stability}. A comparison between the training curves with (orange) and without (blue) the stabilizing term is presented. Our findings demonstrate that the stabilizing term enhances the stability of the training process and leads to the generation of higher quality examples.

\begin{proposition}
    The training processes of elastic discriminator $\max \mathcal{L}_{D}$ (if $\varepsilon >$ a critical value) and generator $\min \mathcal{L}_{G}$ are stable. 
\end{proposition}
The details of the proof are provided in Appendix \ref{Appendix3}.

\begin{figure}[!hbtp]
    \centering
    \begin{subfigure}
    \centering
        \includegraphics[width=0.23\textwidth]{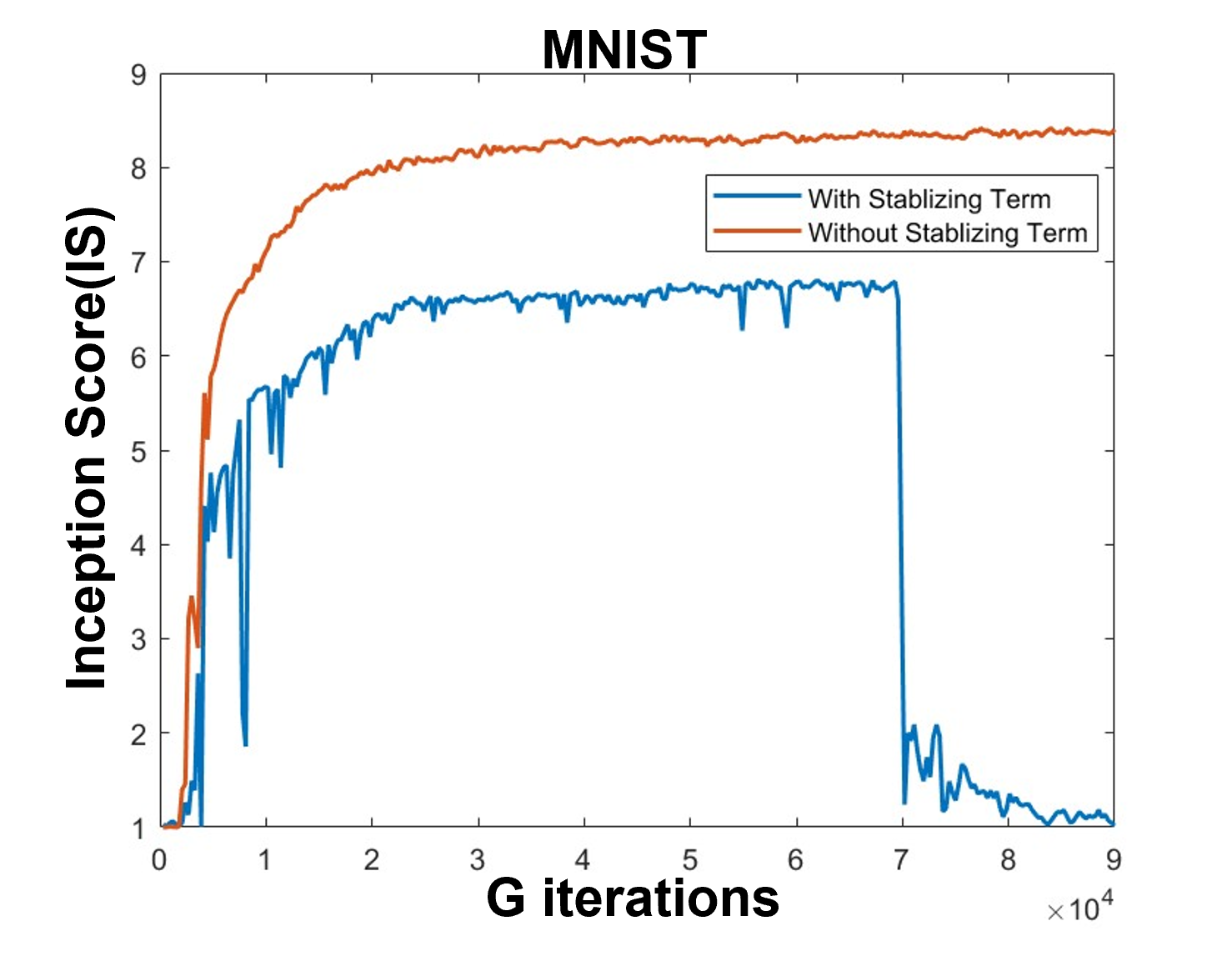}
    \end{subfigure}   
    \begin{subfigure}
    \centering
        \includegraphics[width=0.23\textwidth]{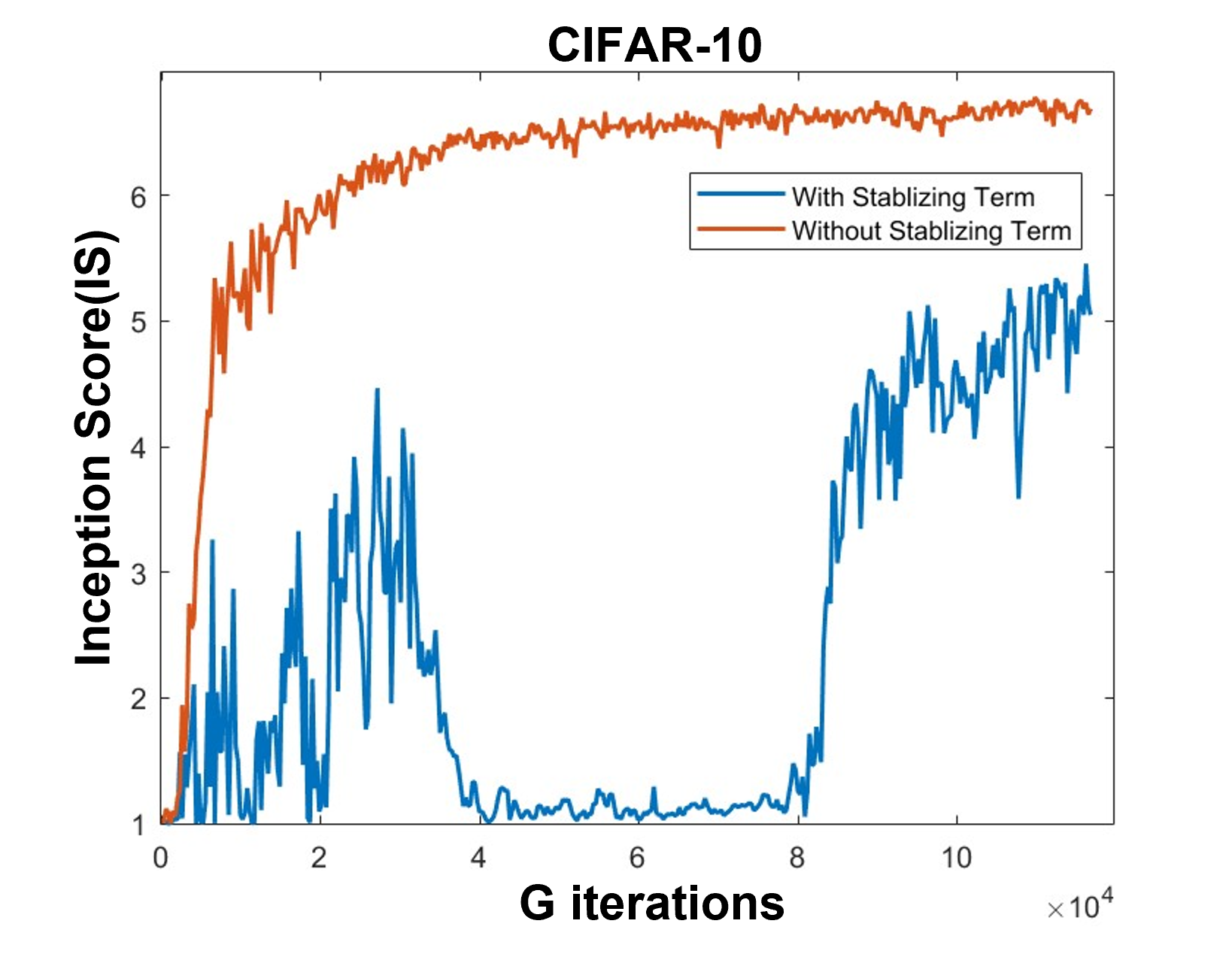}
    \end{subfigure}
    \caption{Training curves on MNIST and CIFAR-10 of EIEG GAN.}
    \label{Fig:EIEG:stability}
\end{figure}

%\subsection{(3).Avoiding mode collapse.}
\paragraph{(3) Stable training process} 
In addition to stabilizing the training process, the additional stabilizing term in $\mathcal{L}_{D}$ also has a mutual exclusion effect on the data in the feature space, which prevents mode collapse and improves the quality of image generation. This effect is due to the repulsive nature of the stabilizing term, which helps to maintain diversity among the generated samples by preventing them from collapsing onto a few dominant modes in the feature space. Therefore, the extra stabilizing term not only improves the stability of the training process but also enhances the quality of the generated images.

\begin{figure}[!hbtp]
    \centering
    \begin{subfigure}
        \centering
        \includegraphics[width=0.23\textwidth]{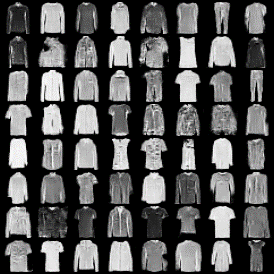}
    \end{subfigure}
    \begin{subfigure}
        \centering
        \includegraphics[width=0.23\textwidth]{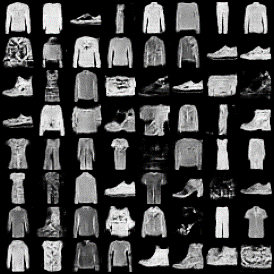}
    \end{subfigure}

    \begin{subfigure}
        \centering
        \includegraphics[width=0.23\textwidth]{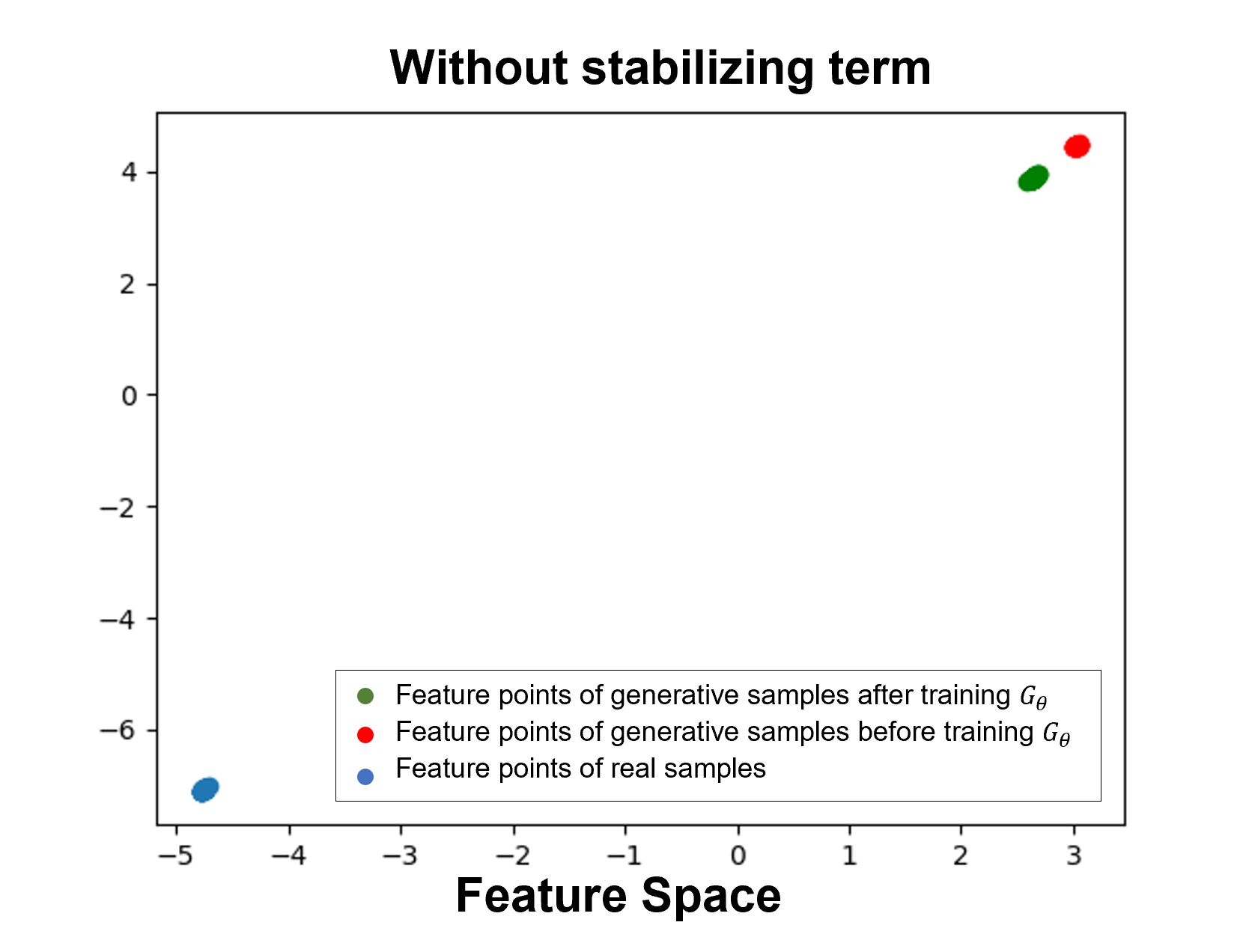}
    \end{subfigure}
    \begin{subfigure}
        \centering
        \includegraphics[width=0.23\textwidth]{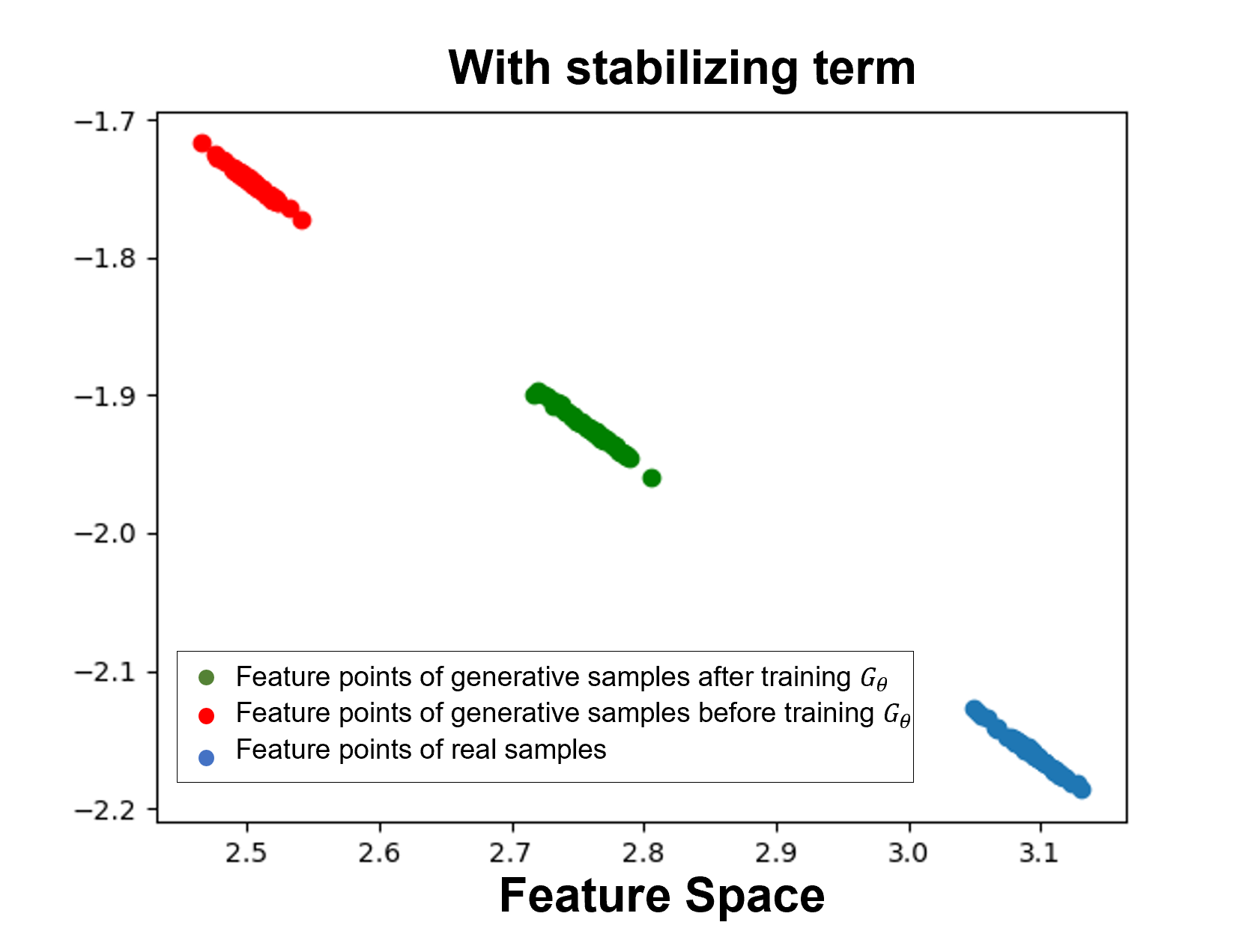}
    \end{subfigure}
    \caption{Generative samples on FashionMNIST and the corresponding 2-dimensional feature space.}
\label{Fig:extraterm feature}
\end{figure}

As illustrated in Fig.\ref{Fig:extraterm feature}, we conduct experiments on generating samples from the FashionMNIST dataset. We map the data onto a 2-dimensional feature space to provide a more intuitive visualization. In the left column, we train $D_{\phi}$ without the stabilizing term. Observing the collapsed data points in the feature space, we note that this results in the generated images exhibiting a lack of diversity and variation. On the other hand, the right column presents the case where $D_{\phi}$ is trained with $\mathcal{L_{D}}$. In this feature space, data distribution is denser, and no data collapse occurs, leading to a diversity of generative images.

\section{Comparisons with related works}
\paragraph{Generative adversarial network}
A variety of GANs can be developed by utilizing different probabilistic metrics to estimate the discrepancy between two distributions. %The role of the discriminator network, denoted as $D_{\phi}$, is to learn how to approximate the chosen metric between the distributions, such as the JS divergence in the original GAN \cite{ref1} and the Wasserstein distance in WGAN \cite{ref3}. 
For example, original GAN \cite{ref1} uses the JS divergence and
WGAN \cite{ref3} adopts the Wasserstein distance.
Nevertheless, due to the inherent min-max framework, GAN suffers from the drawback of unstable training \cite{refSNGAN,refUnstable}.

We have put forth a novel approach in which the metric between distributions is determined by the EIEG metric. 
By incorporating the feature transformation network $D_{\phi}$, our proposed EIEG GAN algorithm utilizes the EIEG metric to measure the distance between two distributions in the latent space. 
With an extra stabilizing term in the loss function of the transformation mapping, our proposed model exhibits enhanced stability in training compared to other GAN methods(see Fig.~\ref{Fig:stability}).

\paragraph{MMD-GAN}
Our model bears similarity to the loss function employed in MMD-GAN \cite{ref4,ref5}. Specifically, MMD-GAN utilizes the square of the Maximum Mean Discrepancy (MMD) distance, which is defined as the norm of the difference between the mean embeddings of two distributions in a reproducing kernel Hilbert space.
\begin{equation}
    \begin{aligned}
        & M_k(\mathbb{P}, \mathbb{Q}) =\left\|\mu_{\mathbb{P}}-\mu_{\mathbb{Q}}\right\|_{\mathcal{H}}^2 \\
        &=\mathbb{E}_{\mathbb{P}}\left[k\left(x, x^{\prime}\right)\right]-2 \mathbb{E}_{\mathbb{P}, \mathbb{Q}}[k(x, y)]+\mathbb{E}_{\mathbb{Q}}\left[k\left(y, y^{\prime}\right)\right].
    \end{aligned}
\end{equation}
In MMD-GAN, a Gaussian kernel  $k\left(x, x^{\prime}\right)=\exp \left(\left( - \|x-x^{\prime}\right\|^2\right)$ is used and the Gaussian kernel only consider local information of the distribution.
%The optimization process aims to discover a manifold with stronger signals for MMD two-sample test.

Our model can be regarded as using a different kernel defined as $k\left(x, x^{\prime}\right)=\frac{1}{\left|x-x^{\prime}\right|^{n-1}}$, where $n$ represents the dimension of $x$. This kernel is selected to align with the long-range properties of elastic interaction energy.% The long-range nature of the elastic interaction energy 
Moreover, our related loss function allows efficient optimization, contributing to improved performance compared to MMD-GAN. 
And with the stabilizing term in the loss function $\mathcal{L}_{D}$, our training process is much more stable.

\paragraph{Possion flow generative model}
Recently, the Diffusion model \cite{ref8,ref9,ref10,ref11,ref101} has gained considerable attention and has been extensively developed. These methods employ stochastic ordinary differential equations to generate samples that move dynamically toward areas where the target distribution is satisfied. 
To achieve this, neural networks are utilized to learn the force field in the dynamic system, for example, score-based generative models use neural network to learn the score (gradient of the log probability density of the target distribution). 
We have observed that the PFGM \cite{ref10} diffusion model uses electronic potential as the metric between the two distributions, similar to our proposed EIEG loss. The PFGM framework is based on the diffusion model and considers only the interaction term of the energy, which can lead to sample collapse. To overcome this issue, they map a uniform distribution on a high-dimensional hemisphere into the data distribution and follow the diffusion model's framework for sampling.

In our proposed model, we utilize self-potential energy, which causes an exclusive force between samples and results in generative samples that do not collapse, as illustrated in Fig.~\ref{Fig:collapse}. Our generative model is mainly based on the GAN framework. The GAN framework offers an advantage over the Diffusion model as it samples through neural networks, leading to faster generation speed and lower operational costs. Therefore, we propose the EIEG GAN model as a more effective and stable generative model.

\FloatBarrier

\section{Experiment}

\paragraph{Datasets}
We train EIEG GAN for image generation on MNIST \cite{ref15}, FashionMNIST \cite{ref20}, CIFAR-10 \cite{ref16} and CelebA \cite{ref17} datasets. MNIST and FashionMNIST have a set of 50k examples as $28 \times 28$ bilevel images. CIFAR-10 has a set of 50k examples as $32 \times 32$ color images. CelebA has a set of over 200k celebrity images as $160 \times 160$. All images are resized to $32 \times 32$ and rescaled so that pixel values are in $[0,1]$. 

\paragraph{Evaluation metrics}
Inception Score(IS)  \cite{ref21} was used for quantitative evaluation. Following\cite{ref23}, we trained a classifier on MNIST, FashionMNIST, and CIFAR-10 using a pre-activation ResNet-18\cite{ref24}.

\paragraph{Network architecture}
We use the neural network architecture of DCGAN \cite{ref18} to set its generator $G_{\theta}$ and  replace the output layer of the discriminator to $n$ dimensional space as our feature transformation network $D_{\phi}$.

\paragraph{Hyper-parameters}
We use Adam \cite{refadam} for generator with the learning rate of 0.0001 and feature mapping transformation with the learning rate of 0.00001. We set the output size of the feature transformation mapping network to $n = 2$. The batch size is set to be $B=64$ for all datasets.

\subsection{Results}
\paragraph{25-Gaussians Example} We conduct experiments on the 25 Gaussians\cite{ref8,ref13,ref14} generation task.The 25-Gaussians dataset is a 2D toy data generated by a mixture of 25 two-dimensional Gaussian distributions. We train GAN, WGAN-GP, and our EIEG GAN, whose networks in the models are parameterized by multilayer perceptions, with two hidden layers and LeakyReLu nonlinear.

The training results are shown in Fig.\ref{Fig:25-Gaussians}. The Vanilla GAN \cite{ref1} exhibits mode collapsing, and many generated samples are collapsing. For WGAN-GP \cite{ref2}, it performs even worse and fails to capture the Gaussian modes. However, our model EIEG GAN successfully captures all the 25 Gaussian modes and can approximate the distribution well.

\begin{figure}[!hbtp]
    \centering
    \begin{subfigure}
        \centering
        \includegraphics[width=0.23\textwidth]{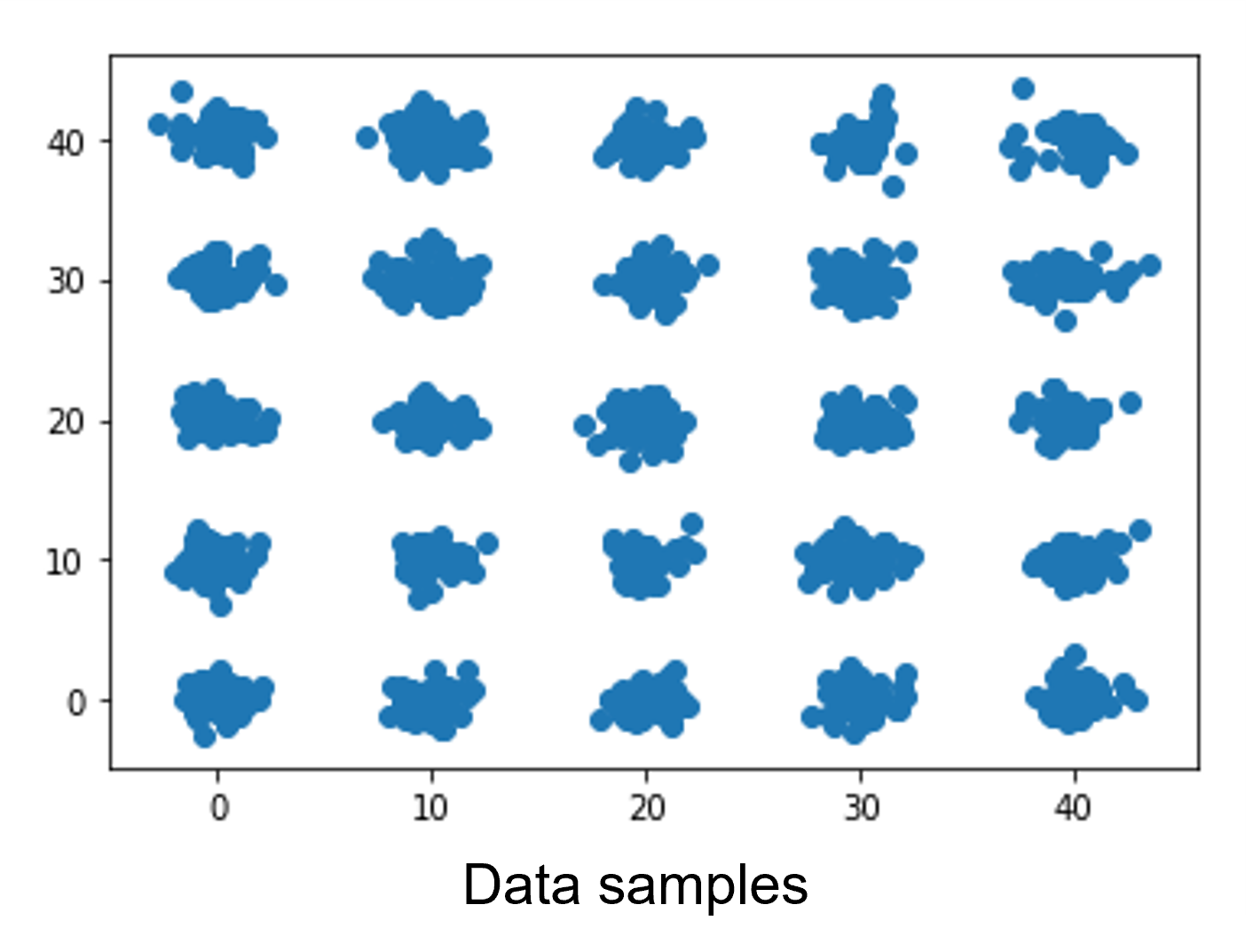}
    \end{subfigure}
    \begin{subfigure}
        \centering
        \includegraphics[width=0.23\textwidth]{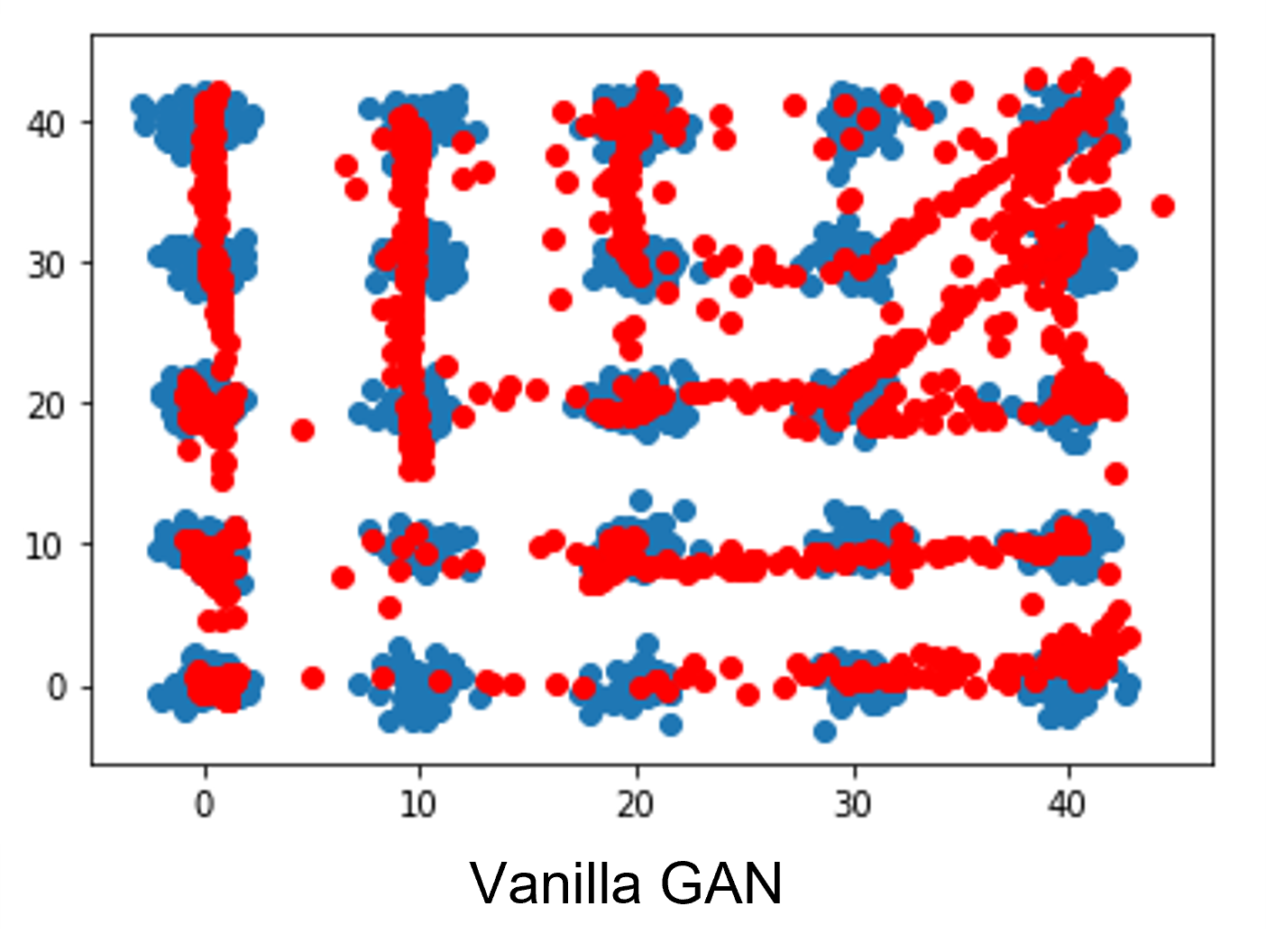}
    \end{subfigure}

    \begin{subfigure}
        \centering
        \includegraphics[width=0.23\textwidth]{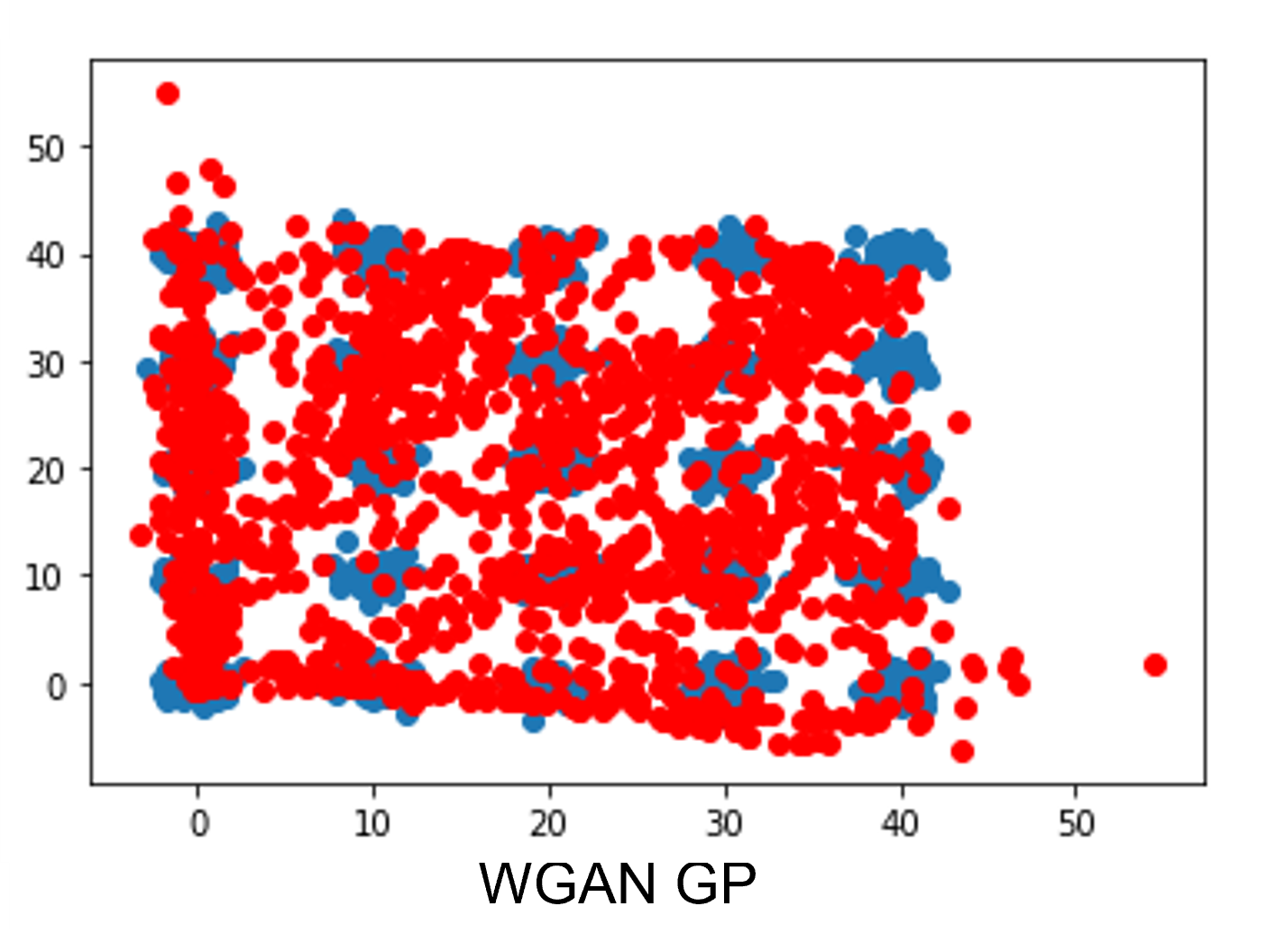}
    \end{subfigure}
    \begin{subfigure}
        \centering
        \includegraphics[width=0.23\textwidth]{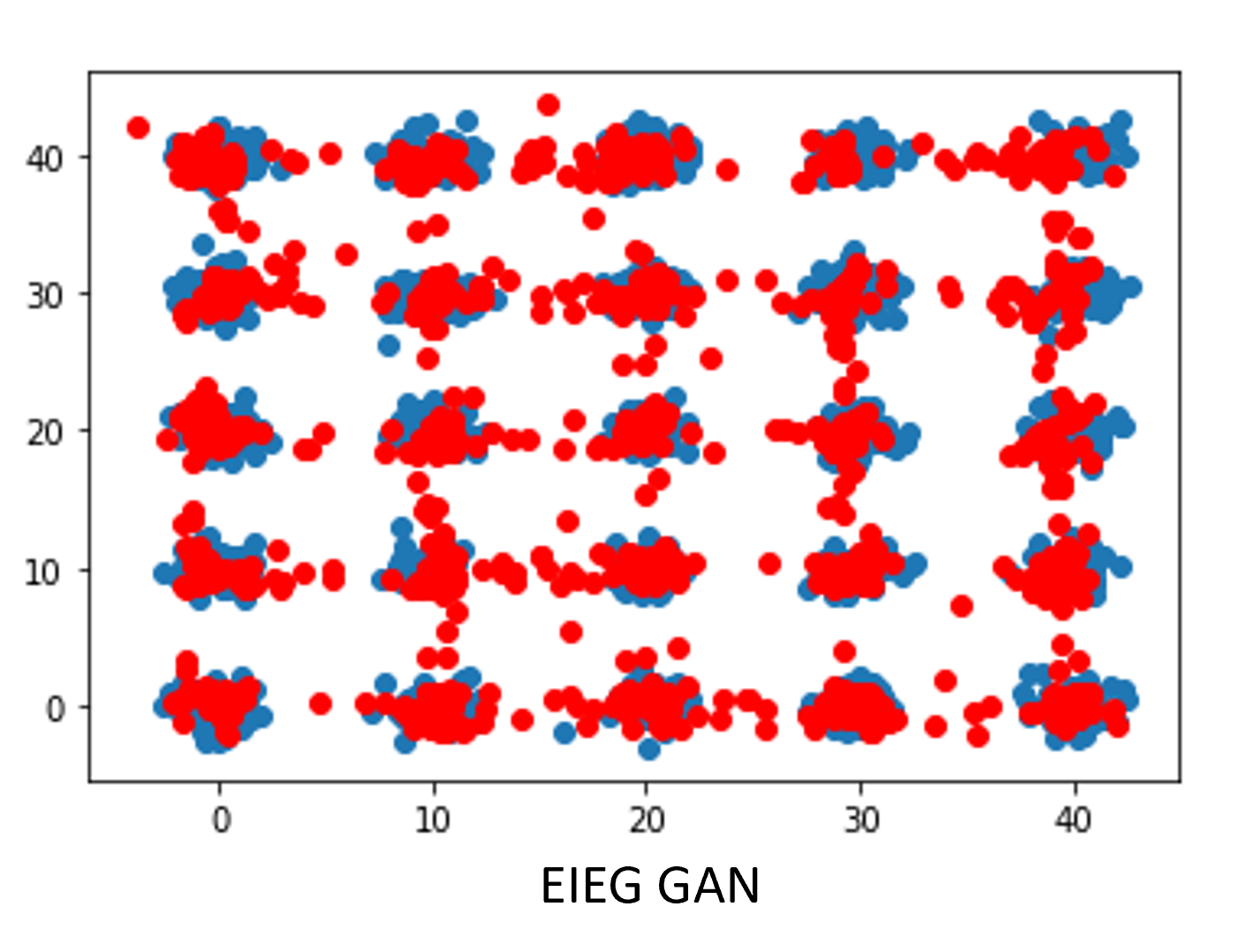}
    \end{subfigure}
    \caption{The 25-Gaussians example. We show the true data samples,the generated samples from vanilla GANs,the generated samples from WGAN-GP, and the generated samples from our EIEG GAN.}
\label{Fig:25-Gaussians}
\end{figure}

\paragraph{Image generation} In Fig.\ref{Fig:Image_generation}, we show uncurated samples generated from our EIEG GAN for MNIST, FashionMNIST, CIFAR-10, and CelebA. Here, we only use the most basic DCGAN network structure without any technique treatments. As shown by the Table \ref{Table: IS Score}, our generated images have higher or comparable quality to those standard GAN-based models\cite{ref1,ref18,ref3,ref2,ref4}. We provide additional details on model architecture and settings in Appendix \ref{Appendix4}.

\begin{figure}[!hbtp]
    \centering
    \begin{subfigure}
        \centering
        \includegraphics[width=0.23\textwidth]{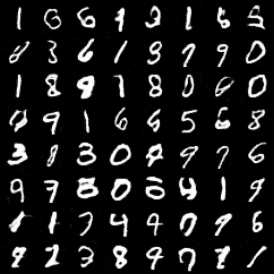}
    \end{subfigure}
    \begin{subfigure}
        \centering
        \includegraphics[width=0.23\textwidth]{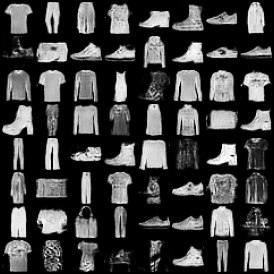}
    \end{subfigure}

    \begin{subfigure}
        \centering
        \includegraphics[width=0.23\textwidth]{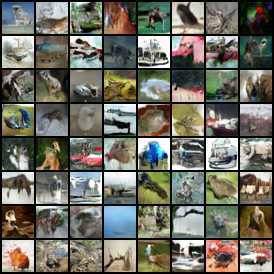}
    \end{subfigure}
    \begin{subfigure}
        \centering
        \includegraphics[width=0.23\textwidth]{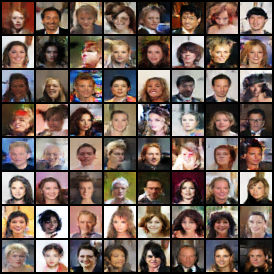}
    \end{subfigure}
    \caption{Uncurated samples on MNIST, FASHIONMNIST, CIFAR-10, and CelebA datasets. }
\label{Fig:Image_generation}
\end{figure}

\paragraph{Quantitative analysis}
To assess the quality and diversity of the generated samples, we quantify the inception score on CIFAR-10 images. To this end, we employ the basic network architecture of DCGAN without any additional training techniques. We compare our model, EIEG-GAN, with several representative extensions of GANs. The results demonstrate that EIEG GAN performs comparably to other representative GANs. Specifically, Table \ref{Table: IS Score} presents the inception scores for 50K generated samples by different representative GAN extensions trained on FashionMNIST and CIFAR-10 datasets.

\begin{table}[!hbtp]
\begin{center}
\begin{tabular}{ccc}
\hline
Method& FashionMNIST& CIFAR-10\\
\hline
PixelCNN \cite{refpixelCNN} &- &4.60\\
WGAN \cite{ref4}& 7.78&5.88\\
WGAN-GP \cite{refSNGAN}&7.97&6.68\\
DCGAN \cite{ref18}& 8.05&6.16\\
MMD GAN \cite{ref4}&- &6.17\\
EIEG GAN-2& 8.44&6.75\\
EIEG GAN-16& 8.54&\textbf{7.02}\\
EIEG GAN-32& \textbf{8.75}&-\\
\hline
\end{tabular}
\end{center}
\caption{Inception scores $\uparrow$ for FashionMNIST and CIFAR-10. And EIEG GAN-$n$ 
indicates that the size of the feature space is $n$.}
\label{Table: IS Score}
\end{table}

\FloatBarrier
\begin{comment}

\begin{figure}[H]
\centering  %图片全局居中
\subfigure[EIEG-GAN CeleA]{
\label{Fig.sub.1}
\includegraphics[width=0.3\textwidth]{RKHS_CELE_R01.png}}
\subfigure[EIEG-GAN-GP CeleA ]{
\label{Fig.sub.2}
\includegraphics[width=0.3\textwidth]{RKHS_CELE_test09_3.png}}
\subfigure[EIEG-GAN-GP-RE CeleA]{
\label{Fig.sub.2}
\includegraphics[width=0.3\textwidth]{RKHS_CELE_test09_3.png}}
\caption{Comparison of different kinds of EIEG-GAN. }
\label{Fig.main}
\end{figure}

\begin{figure}[H]
\centering  %图片全局居中
\subfigure[EIEG-GAN-GP MNIST]{
\label{Fig.sub.1}
\includegraphics[width=0.3\textwidth]{RKHS_CELE_test09_1.png}}
\subfigure[EIEG-GAN-GP CIFAR-10]{
\label{Fig.sub.2}
\includegraphics[width=0.3\textwidth]{RKHS_CIFAR_R01_I25_GP5_2.png}}
\subfigure[EIEG GAN-GP CeleA]{
\label{Fig.sub.2}
\includegraphics[width=0.3\textwidth]{RKHS_CELE_test09_3.png}}
\caption{Hyperparameter: cut-off radius $R = 0.1$, Feature map iteration 5, $\lambda = 5$. Feature map size: $64$.}
\label{Fig.main}
\end{figure}

\end{comment}

\FloatBarrier
\subsection{Stability and effectiveness}
Our EIEG GAN model offers a significant advantage in terms of stability during training, which is attributed to the loss function of our feature transformation network $D_{\phi}$. To support this claim, we present a comparative analysis of the training progress of our model with that of other GAN models on FashionMNIST and CIFAR-10 datasets in Fig.\ref{Fig:stability}. As shown in the figure, the learning curves of EIEG GAN are notably smoother, indicating greater training stability. The inception scores steadily increase until 10k iterations (approximately 10 epochs) on FashionMNIST and 40k iterations (nearly 30 epochs) on CIFAR-10, indicating that our model converges rapidly. On the other hand, the training of DCGAN on FashionMNIST is found to be highly unstable, while WGAN shows instability in the training process on both datasets. WGAN-GP takes longer to converge on FashionMNIST, while its training on CIFAR-10 appears to be unstable. Overall, our model outperforms these representative GAN models in terms of the quality of the generated images and training stability and effectiveness on these two datasets. 

\begin{figure}[!hbtp]
    \centering
    \begin{subfigure}% FashionMNIST
    \centering
        \includegraphics[width=0.5\textwidth]{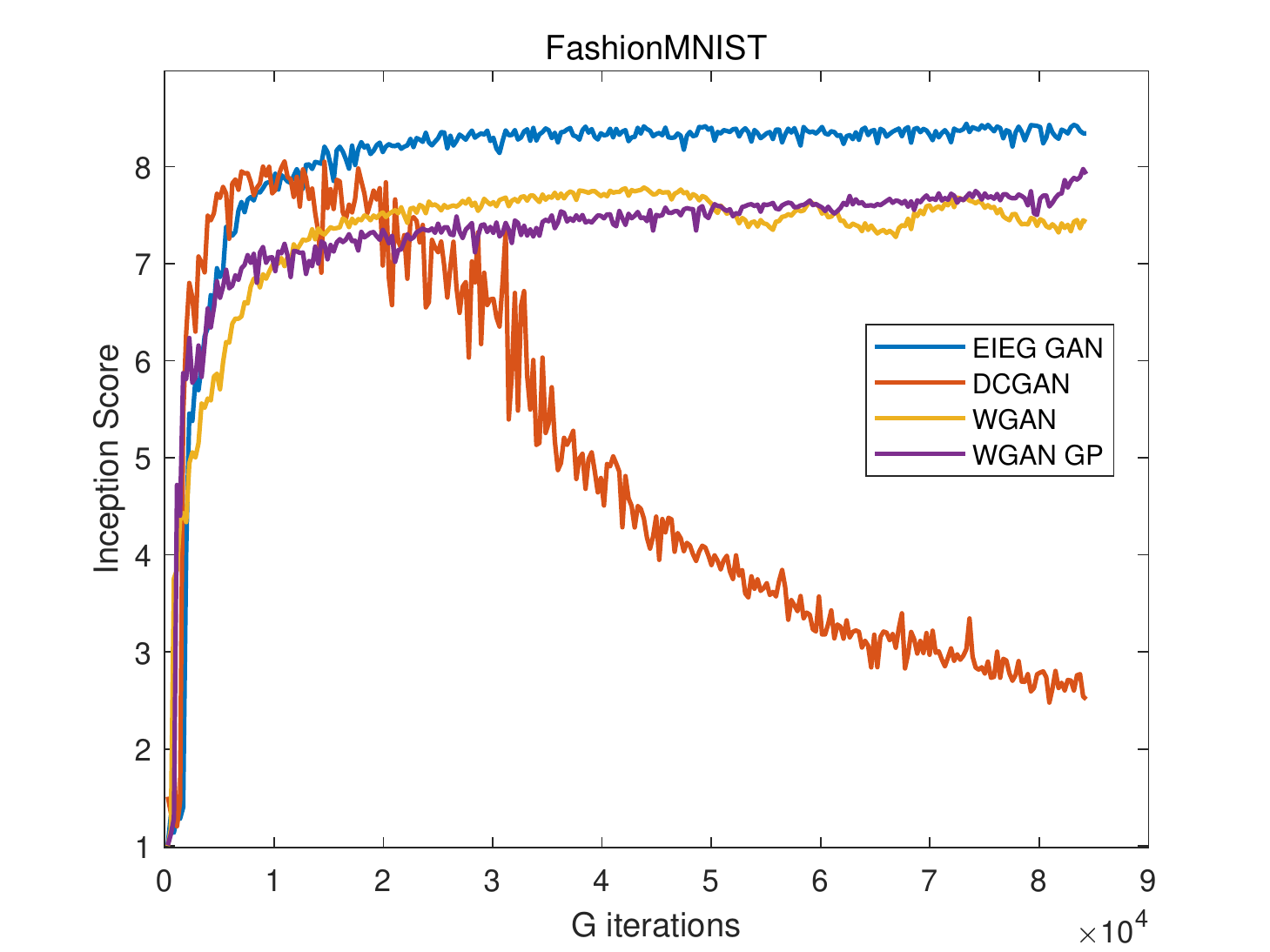}
    \end{subfigure}   
    \begin{subfigure} %[CIFAR10]
    \centering
        \includegraphics[width=0.5\textwidth]{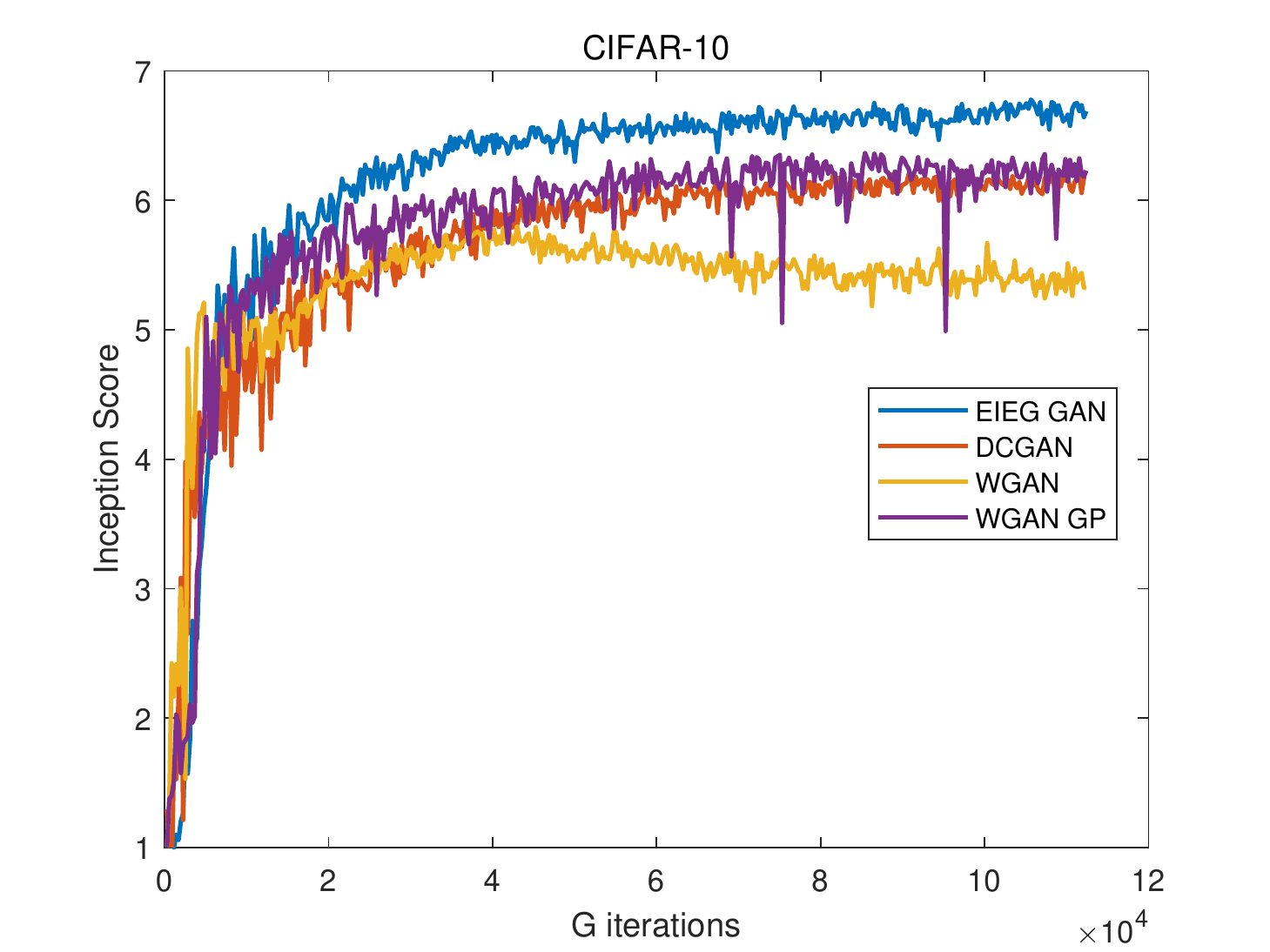}
    \end{subfigure}
    \caption{Traning Curves on FashionMNIST and CIFAR10 of EIEG GAN and different GAN-based models.}
    \label{Fig:stability}
\end{figure}

\subsection{Feature space effect}
Furthermore, as shown in Table \ref{Table: IS Score} our model's performance improves with an increase in the dimension of the feature space. We provide the experimental results in Appendix \ref{Appendix5_2}. This is likely because higher-dimensional feature spaces can capture more complex and subtle patterns in the data, which can lead to more realistic and varied generated samples. Specifically, our experiments demonstrate that EIEG GAN generates more diverse and realistic samples than the other GAN-based models, while also exhibiting improved training stability. 
 Higher generation quality is expected with increasing the dimension of the feature space. These findings highlight the superior performance of EIEG GAN and underscore its potential to advance the state-of-the-art in generative modeling.

%\FloatBarrier

\section{Conclusions}

We present a novel approach, namely EIEG GAN, for generative modeling which leverages the elastic interaction energy (EIE) to capture global distribution information and promote diverse sample generation, while mitigating the problem of mode collapse. One of the key innovations of our approach is the usage of a feature transformation network as a discriminator, which maps high-dimensional data into a lower dimensional feature space. Additionally, we introduce a stabilizing term to the loss function of the feature transformation network to enhance the stability of GAN-based training. Our experimental results demonstrate that the proposed EIEG GAN approach surpasses several standard GAN-based models with respect to both sample diversity and training stability. These findings highlight the potential of the EIEG GAN approach to advance the state-of-the-art in GAN-based generative modeling.

\section*{Acknowledgements}

This work was supported by the Project of Hetao Shenzhen-HKUST Innovation Cooperation Zone HZQB-KCZYB-2020083.

{\small
\bibliographystyle{ieeetr}
\bibliography{egpaper_for_review}
}

\newpage

\begin{appendices}
\noindent \title{\large\textbf{Appendix}}
\section{EIEG Metric} \label{Appendix1}

Consider two probability density functions $p(\mathbf{x}), q(\mathbf{x}) : \mathbb{R}^{n} \rightarrow \mathbb{R}$, the EIEG metric between these two probability density function is
\begin{equation}
 \begin{aligned}
&E[p(\mathbf{x}),q(\mathbf{x}) ] \\&=\int_{\mathbb{R}^{n}}  \int_{\mathbb{R}^{n}}(p(\mathbf{x}) - q(\mathbf{x})) \cdot \frac{(p(\mathbf{y}) - q(\mathbf{y}))}{r^{n-1}} d \Omega_\mathbf{x} d \Omega_\mathbf{y},
 \end{aligned}
 \label{Eqn:EIEG metric}
\end{equation}
where $r = \| \mathbf{x} - \mathbf{y}\|$.

\begin{proposition}[in main paper]

(i) Given two probability distribution $\mathbb{P}$ and $\mathbb{Q}$, we have $E[\mathbb{P},\mathbb{Q}] \geq 0$ and $E[\mathbb{P},\mathbb{Q}] = 0 \Leftrightarrow \mathbb{P} = \mathbb{Q}$.

 (ii) Let $\left\{\mathbb{P}_n\right\}$ be a sequence of distributions, and $\mathbb{P}_{n}$ is the corresponding probability density function. Considering $n \rightarrow \infty$, $E[\mathbb{P}_n,\mathbb{P}_{data}] \rightarrow 0 \Longleftrightarrow \|\mathbb{P}_n - \mathbb{P}_{data}\|_{\text{semi-}{H^{-\frac{1}{2}}}} \rightarrow 0$.

\end{proposition}

\begin{proof}
For the convenience of illustration, we prove the two-dimensional example here, and the proof of the high-dimensional case is similar. 

Fourier transform for a function $f(x,y)$ in two dimensional space is
\begin{equation}
\mathcal{F}(f(x,y))(m, n)=\int_{-\infty}^{\infty} \int_{-\infty}^{\infty} f(x, y) e^{-i 2 \pi(m x+n y)} d x d y,
\end{equation}
where $m,n$ are frequencies in the Fourier space. The convolution between two functions $f(\mathbf{x})$ and $g(\mathbf{x})$ is
\begin{equation}
(f * g)(\bm{t}):=\int_{\mathbb{R}^2} f(\bm{\tau}) g(\mathbf{t}-\bm{\tau}) d \bm{\tau}.
\end{equation}

For a probability density function $p(\mathbf{x}): \mathbb{R}^{2} \rightarrow \mathbb{R}$, consider the energy
\begin{equation}
\begin{aligned}
E[p(\mathbf{x})] =&\int_{\mathbb{R}^{2}}  \int_{\mathbb{R}^{2}}p(\mathbf{x}) \cdot \frac{p(\mathbf{y}) }{r} d \Omega_\mathbf{x} d \Omega_\mathbf{y} \\
=&\int_{\mathbb{R}^{2}} p(\mathbf{x}) (p * \frac{1}{r})(\mathbf{x}) d\Omega_\mathbf{x}.
\end{aligned}
\end{equation}

Using Parseval's identity, we have
\begin{equation}
\begin{aligned}
    E[p(\mathbf{x})] =&\int_{\mathbb{R}^{2}} \overline{\mathcal{F}(p(\mathbf{x}))} \mathcal{F}((p * \frac{1}{r})(\mathbf{x}))d\Omega_\mathbf{x}\\
    =&\sum_{|\bm{\xi}|} \overline{\hat{p}(\bm{\xi})}\hat{p}(\bm{\xi})\mathcal{F}(\frac{1}{r})\\
    =& \sum_{|\bm{\xi}|}  \frac{1}{|\bm{\xi}|}\overline{\hat{p}(\bm{\xi})}\hat{p}(\bm{\xi})\\
    =& \sum_{m,n}\frac{1}{\sqrt{m^2 + n^2}}|\hat{p}(m,n)|^2,
\end{aligned}
\end{equation}
hwhere $\hat{p}(\mathbf{\xi})$ is the Fourier transform of $p(\mathbf{x})$, $\mathbf{\xi} = (m,n)$ is the Fourier coefficients. This is a semi-$H^{-\frac{1}{2}}$ norm.

Thus for the EIEG metric we have 
\begin{equation}
\begin{aligned}
    E[p(\mathbf{x}),q(\mathbf{x}) ] \triangleq& E[p(\mathbf{x})-q(\mathbf{x}) ] 
    \\=& \sum_{m,n}\frac{1}{\sqrt{m^2 + n^2}}|\hat{p} - \hat{q}|^2
\end{aligned}
\end{equation}
which is easy to see that for any two smooth distribution density functions $p(\mathbf{x})$ and $q(\mathbf{x})$
\begin{equation}
  E[p,q] \geq 0, and \quad E[p,q] = 0 \Leftrightarrow p = q.
\end{equation}
Moreover, for a sequence of smooth distributions $p_{n}$,
\begin{equation}
\begin{aligned}
    &\lim_{n \rightarrow \infty} E[p_{n}(\mathbf{x}),p(\mathbf{x}) ] \rightarrow 0 \\ \Longleftrightarrow & \lim_{n \rightarrow \infty} \|p_n - p \|^{2}_{semi-H^{-\frac{1}{2}}} \rightarrow 0 .
\end{aligned}
\end{equation}
\end{proof}

\paragraph{Strong long-range interaction.} 
From Eqn.(\ref{Eqn:EIEG metric}) we can let $p(\mathbf{x}) = \delta_{\mathbf{x_0}}(\mathbf{x})$ and $q(\mathbf{x}) = \delta_{\mathbf{y_0}}(\mathbf{x})$, and we can get the elastic interaction energy between two sample points $\mathbf{x_0}$ and $\mathbf{y_0}$,
\begin{equation}
E_{int}[\delta_{\mathbf{x_0}}(\mathbf{x}),\delta_{\mathbf{y_0}}(\mathbf{x}) ] =\\
-\frac{2}{\|\mathbf{x_0} - \mathbf{y_0}\|^{n-1}}.
\label{Eqn:EIEG point metric int}
\end{equation}
As noted in the main paper, the elastic interaction energy (EIE) between two sample points exhibit long-range behavior, with the energy being inversely proportional to the $(n-1)$th powers of the distance between them. As the distance approaches infinity, this decay is very slow. The attractive force between two sample points is the negative gradient of the energy $\nabla_{\mathbf{x_0}}E[\delta_{\mathbf{x_0}}(\mathbf{x}),\delta_{\mathbf{y_0}}(\mathbf{x})]$. It can be observed that the interaction force between two sample points exhibits the asymptotic property 
\begin{equation}
    \mathbf{f} \propto \frac{1}{r^n},
\end{equation} 
where $r$ is the distance from the samples. Consequently, elastic interaction results in a strong attractive force between samples from the data distribution $\mathbb{P}_{data}$ and the generated distribution $\mathbb{P}_{\theta}$, as the generative samples gradually approach the data samples.

In Eqn.(\ref{Eqn:EIEG point metric int}), we only consider the interaction term between $\mathbf{x_0}$ and $\mathbf{y_0}$. The self-interaction terms of $\mathbf{x_0}$ and $\mathbf{y_0}$ are singularities. To address this issue, we set a cut-off in the EIEG loss in the main paper section 2.

\section{Scattered distribution in high-dimensional space and its consequence.} \label{Appendix2}
The findings presented in Figure 5 of the main paper suggest that using only the generator network $G_{\theta}$ under the EIEG loss is insufficient to generate high-quality samples directly from high-dimensional datasets. This is primarily because the distribution of data points in the high-dimensional data space is scattered.

In this section, we empirically demonstrate that the distribution of high-dimensional datasets exhibits a scattered structure. We perform sampling in the data space through an ordinary differential equation (ODE).

Consider the EIEG metric between the data distribution $\mathbb{P}_{data}$ and generated distribution $\mathbb{P}_{\theta}$. In order to minimize the EIEG metric between them $\min E[P_{data},P_{\theta}]$ ($P_{data}$ and $P_{\theta}$ stands for the corresponding distribution density function), we can have a Wasserstein gradient flow for the probability density function $P_{\theta}$ \cite{refwflow}:

\begin{equation}
\begin{aligned}
      \frac{\partial P_{\theta} }{\partial t} =& \nabla \cdot\left(P_{\theta} \nabla \frac{\delta E}{\delta P_{\theta}}\right) \\
      =& \nabla \cdot\bigg(P_{\theta} \nabla \bigg(2\mathbb{E}_{y \sim \mathbb{P}_{\theta}}(\frac{1}{r^{n-1}}) - 2\mathbb{E}_{y \sim \mathbb{P}_{data}}(\frac{1}{r^{n-1}})\bigg)\bigg).
\end{aligned}
\end{equation}

By the Fokker Planck Equation \cite{reffokker} we can get the corresponding evolution equation for the generative samples is:
\begin{equation}
    dX_{t} = \mathbb{E}_{Y \sim \mathbb{P}_{data}}\frac{\vec{Y}-\vec{X}}{r^{n+1}}dt - \mathbb{E}_{Y \sim \mathbb{P}_{\theta}}\frac{\vec{Y}-\vec{X}}{r^{n+1}}dt ,
\label{eq:dynamic}
\end{equation}
where the first term stands for the attractive effect between the data samples and generated samples and the second term stands for the repulsive effect between the generated samples.

Thus using the cut-off as in the main paper, we can get the final evolution dynamics for a generated sample $X_{i}$
\begin{equation}
    dX_{i,t} =\left(
     M_1 \frac{1}{N_1}\sum_{j=1}^{N_1}{\overrightarrow{f_{ij,int}}} - M_2\frac{1}{N_2}\sum_{j=1}^{N_2} \overrightarrow{f_{ij,self}}
    \right) dt
\label{Eqn: sampling}
\end{equation}
where
\begin{equation}
    {\overrightarrow{f_{i j}}}=\begin{cases}
\frac{\vec{Y}_{j}-\vec{X}_{i}}{r_{i j}^{n+1}} \quad &r_{i j} \geqslant R \\
\\\frac{\left(\vec{Y}_{j}-\vec{X}_{i}\right)}{R^{n+1}}  \quad &r_{i j}<R
\end{cases}.
\end{equation}
Here $\overrightarrow{f_{ij,int}}$, samples $\vec{Y}_{j} \sim \mathbb{P}_{data}$ and in $\overrightarrow{f_{ij,self}}$, samples $\vec{Y}_{j} \sim \mathbb{P}_{\theta}$, $M_1$ and $M_2$ are mobilities of the dynamics and $N_1$ and $N_2$ are numbers of the samples.

\paragraph{Experiments.} We use Eqn.(\ref{Eqn: sampling}) to perform the generative sampling directly. The initial state for the generative sample is $X_{i,0} \sim \mathcal{N}(0,I)$. We test on MNIST \cite{ref15}, CIFAR-10 \cite{ref16}, CelebA \cite{ref17}. The hyperparameters setting can be found in Table\ref{Table: Hyperpara_ODE}. Results can be found on Fig.\ref{Fig: GEN_ODE}.

\begin{table}[!hbtp]
\begin{center}
\begin{tabular}{cc}
\hline
Hyperparameters& settings\\
\hline
Cut-off $R$  & 1\\
Mobility $M_1$ & 100 \\
Mobility $M_2$ & 50\\
Step size $\Delta t$ & 0.1\\
Batch size $N_1$ & 64\\
Batch size $N_2$& 64\\
Total steps $T$ & 100000\\
\hline
\end{tabular}
\end{center}
\caption{Hyperparameters for directly using ODE Eqn.(\ref{Eqn: sampling}) to generate samples.}
\label{Table: Hyperpara_ODE}
\end{table}

As illustrated in Fig.\ref{Fig: GEN_ODE}, most of the generated samples exhibit blurring. Our experiments have been executed for a sufficiently long duration, and the dynamic system has achieved a state of equilibrium. We can exclude the possibility of blurred results stemming from the lack of convergence. Furthermore, we observe a few clear samples which can be found in the original dataset. We believe that the reason behind these phenomena can all be attributed to the scattered distribution of data in the high-dimensional data space. Specifically, for these few clear generated samples, we attribute their clarity to the fact that they were directly attracted to a specific data sample, while other data samples that are far away from have only very small attractive force to these few clear generated samples. For the majority of the blurry generated samples, the scattered distribution of data in high-dimensional space prevented them from being attracted to a specific mode of distribution, which may lead to the vague phenomenon.

\textbf{Remark:} To address the issue that data distribution is scattered in high-dimensional data space, as presented in the main paper, we propose a feature transformation network to map the high-dimensional data to a feature space with a more compact and smooth distribution. This transformation enables our EIEG model to more accurately approximate the underlying distribution in feature space.

\begin{figure}[!hbtp]
    \centering
    \begin{subfigure}% FashionMNIST
    \centering
        \includegraphics[width=0.23\textwidth]{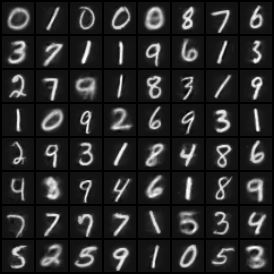}
    \end{subfigure}   
    \begin{subfigure} %[CIFAR10]
    \centering
        \includegraphics[width=0.23\textwidth]{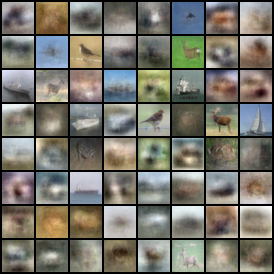}
    \end{subfigure}
    \begin{subfigure} %[CIFAR10]
    \centering
        \includegraphics[width=0.23\textwidth]{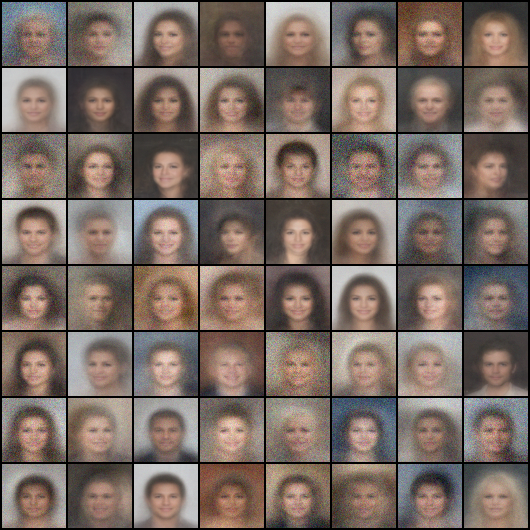}
    \end{subfigure}
    \caption{Generated samples on MNIST, CIFAR10 and CelebA.}
    \label{Fig: GEN_ODE}
\end{figure}  

\section{Stability analysis} \label{Appendix3}
One of our contributions presented in the main paper is that we add a stabilizing term in the loss function of the elastic discriminator $\mathcal{L}_D$ (Eqn.(8) in the main paper). This property is summarized in \textbf{Proposition 2} in the main paper. Here we provide the proof of proposition 2.

\begin{proposition}[in main paper]
    The training processes of elastic discriminator $\max \mathcal{L}_{D}$ (if $\varepsilon >$ a critical value) and generator $\min \mathcal{L}_{G}$ are stable. 
\end{proposition}

\begin{proof}
   In this proof, we will demonstrate the stability of the training process of the generative model $G_{\theta}$, the instability of the training process of the elastic discriminator $D_{\phi}$ without the stabilizing term, and the stability of the training process of $D_{\phi}$, by using the corresponding PDEs of the Wasserstein gradient flow. For the convenience of illustration, we prove the two-dimensional example here, and the proof of the high-dimensional case is similar.

In order to prove the stability, we assume that the distribution of samples generated by the $G_{\theta}(X) \sim \mathbb{P}_{\theta}$ and the distribution of samples after feature transformation $D_{\phi}(X) \sim \mathbb{P}_{\phi}$ should approximate the distribution $\mathbb{P}_{*}$ with density function $P_{*} \equiv C$ $(C>0)$ i.e., $\mathbb{P}_{*}$ is a uniform distribution.
\begin{itemize}
    \item \textbf{The training process of generator $G_{\theta}$ to $\min \mathcal{L}_{G}$ is stable.} For the training process of $G_{\theta}$, the goal is to 
    \begin{equation}
        \min E[P_{*},P_{\theta}],
    \label{Eqn objective G}
    \end{equation}
where $P_{\theta}$ is the distribution density function for the generative samples. Thus the corresponding Wasserstein gradient flow for $P_{\theta}$ is:
\begin{equation}
\begin{aligned}
      \frac{\partial P_{\theta} }{\partial t} =& \nabla \cdot\left(P_{\theta} \nabla \frac{\delta E}{\delta P_{\theta}}\right) \\
      =& \nabla \cdot\left(P_{\theta} \nabla \bigg(2\mathbb{E}_{y \sim \mathbb{P}_{\theta}}(\frac{1}{r}) - 2\mathbb{E}_{y \sim \mathbb{P}_{*}}(\frac{1}{r})\bigg)\right) \\
      =& \nabla \cdot\bigg(P_{\theta} \nabla \bigg(2\int_{\mathbb{R}^2}(\frac{1}{r})P_{\theta}d\Omega_{\mathbf{y}} \\ 
      &\quad \quad \quad \quad \quad - 2\int_{\mathbb{R}^2}(\frac{1}{r})Cd\Omega_{\mathbf{y}}\bigg)\bigg).
\end{aligned}
\label{Eqn: EvoG}
\end{equation}
\\
Knowing that the above PDE has a solution $P_{\theta} = C$, we let $P_{\theta} = u = C+ v$ where $v$ is a small perturbation and $v \ll 1$. Substitute it into Eqn.(\ref{Eqn: EvoG}) and since $v \ll 1 $, we keep only the linear terms of $v$, which is
\begin{equation}
\begin{aligned}
    \frac{\partial v}{\partial t} =& C \Delta  \int_{\mathbb{R}^2}(\frac{1}{r})vd\Omega_{\mathbf{y}} \\
    =&  C \Delta  (\frac{1}{r} * v)(\mathbf{x}).
\end{aligned}
\label{Eqn:Stable1}
\end{equation}
Taking Fourier Transform on both sides of Eqn.(\ref{Eqn:Stable1}), we have
\begin{equation}
    \frac{d \hat{v}}{dt} = -C|\bm{\xi}|^{2}\hat{v}\mathcal{F}(\frac{1}{r}) = -C|\bm{\xi}|\hat{v},
\end{equation}
where $\hat{v}$ is the Fourier transform of $v$ and $\bm{\xi}$ is the Fourier coefficients. 
Thus in the Fourier Space, the solution for the perturbation term $\hat{v}$ is
\begin{equation}
    \hat{v} = e^{-C|\bm{\xi}|t}.
\end{equation}
Thus the perturbations with all $|\bm{\xi}|$ decay, i.e., $|\hat{v}| = |e^{-C|\bm{\xi}|t}| \rightarrow 0$ as $t \rightarrow \infty$. 

Thus, in this case, the equilibrium solution for the Eqn.(\ref{Eqn objective G}), $P_{\theta} = P_{*} = C$ is stable, which can also indicate that the training process of $G_{\theta}$ to $\min \mathcal{L}_{G}$ is stable.

\item \textbf{The training process of the elastic discriminator $D_{\phi}$ is unstable without the stabilizing term.}
For the training process of $D_{\phi}$ under the guidance of  $\min \mathcal{L}_{D}$ without the stabilizing term, the goal is to 
    \begin{equation}
        \max E[P_{*},P_{\phi}],
    \label{Eqn objective D}
    \end{equation}
where $P_{\phi}$ is the distribution density function for the samples after feature transformation $D_{\phi}$. Thus the corresponding Wasserstein gradient flow for $P_{\phi}$ is
  \begin{equation}
 \frac{\partial P_{\phi} }{\partial t} = -\nabla \cdot\left(P_{\phi} \nabla \frac{\delta E}{\delta P_{\phi}}\right).
  \end{equation}  
  
The entire analysis process is consistent with the above $G_{\theta}$ case.

For this case, in the Fourier Space, the solution for the perturbation term $\hat{v}$ is
\begin{equation}
    \hat{v} = e^{C|\bm{\xi}|t}.
\end{equation}
Thus the perturbations with all $|\bm{\xi}|$ decay, i.e., $|\hat{v}| = |e^{C|\bm{\xi}|t}| \rightarrow \infty$ as $t \rightarrow \infty$.

Thus, in this case, the equilibrium solution for the Eqn.(\ref{Eqn objective D}), $P_{\phi} = P_{*} = C$ is unstable, which can also indicate that the training process of $D_{\phi}$ to $\max \mathcal{L}_{D}$ is unstable without the stabilizing term in $\mathcal{L}_{D}$.

\item \textbf{The training process of the elastic discriminator $D_{\phi}$ to $\max \mathcal{L}_D$ is stable.}

From the main paper Equation (8) in section 3, for the training process of $D_{\phi}$, the goal is to 
    \begin{equation}
        \max \widetilde{E}[P_{*},P_{\phi}],
    \label{Eqn objective D_s}
    \end{equation}
where $P_{\phi}$ is the distribution density function for the samples after feature transformation $D_{\phi}$. Thus the corresponding Wasserstein gradient flow for $P_{\phi}$ is
  \begin{equation}
  \begin{aligned}
 \frac{\partial P_{\phi} }{\partial t} &= -\nabla \cdot\left(P_{\phi} \nabla \frac{\delta \widetilde{E}}{\delta P_{\phi}}\right)\\
   =& \nabla \cdot\left(P_{\phi} \nabla \bigg(2\mathbb{E}_{y \sim \mathbb{P}_{*}}\widetilde{e}(r) - 2\mathbb{E}_{y \sim \mathbb{P}_{\phi}}\widetilde{e}(r)\bigg)\right) \\
      =& \nabla \cdot\bigg(P_{\phi} \nabla \bigg(2\int_{\mathbb{R}^2}\widetilde{e}(r)Cd\Omega_{\mathbf{y}}\\ 
      &\quad \quad \quad \quad \quad - 2\int_{\mathbb{R}^2}\widetilde{e}(r)P_{\phi}d\Omega_{\mathbf{y}}\bigg)\bigg),
    \end{aligned}
    \label{Eqn: EvoD}
  \end{equation}  
where $\widetilde{e}(r) = \frac{1}{r} - \varepsilon \frac{1}{r^3}$.

Knowing that the above PDE has a solution $P_{\phi} = C$, we let $P_{\phi} = u = C+ v$ where $v$ is a small perturbation and $v \ll 1$. Substitute it into Eqn.(\ref{Eqn: EvoD}) and since $v \ll 1 $, we keep only the linear terms of $v$, which is
\begin{equation}
\begin{aligned}
    \frac{\partial v}{\partial t} =& -C \Delta  \int_{\mathbb{R}^2}(\frac{1}{r}-\varepsilon \frac{1}{r^3})vd\Omega_{\mathbf{y}} \\
    =&  -C \Delta  \bigg((\frac{1}{r}-\varepsilon \frac{1}{r^3}) * v \bigg)(\mathbf{x}).
\end{aligned}
\label{Eqn:Stable2}
\end{equation}

Taking the Fourier Transform on both sides of Eqn.(\ref{Eqn:Stable2}), we have
\begin{equation}
\begin{aligned}
    \frac{d \hat{v}}{dt} =& C|\bm{\xi}|^{2}\hat{v}\mathcal{F}(\frac{1}{r}-\varepsilon \frac{1}{r^3}) \\=&  C|\bm{\xi}|^{2}\hat{v}\bigg(\mathcal{F}(\frac{1}{r})-\varepsilon \mathcal{F}(-\Delta\frac{1}{r})\bigg)
    \\=&  C|\bm{\xi}|^{2}\hat{v}\bigg(\frac{1}{|\bm{\xi}|}-\varepsilon |\bm{\xi}|\bigg)
    \\=&  C\hat{v}(1-\varepsilon |\bm{\xi}|^2)|\bm{\xi}|,
\end{aligned}
\label{Eqn:26}
\end{equation}
where $\hat{v}$ is the Fourier transform of $v$ and $\bm{\xi}$ is the Fourier coefficients. 
Thus in the Fourier Space, the solution for the perturbation term $\hat{v}$ is
\begin{equation}
    \hat{v} = e^{C(1-\varepsilon |\bm{\xi}|^2)|\bm{\xi}|t}.
\end{equation}
Knowing that in the training of $D_{\phi}$, we always normalize the input data $X$ into a finite domain, i.e., $[-1,1] \times [-1,1]$. And with $\bm{\xi} = (m,n)$, we know that $|\bm{\xi}|$ is bounded and $\min |\bm{\xi}| = \pi$ for $|\bm{\xi}|$ not equal to 0. When $|\bm{\xi}| = 0$, it can be observed from Eqn.(\ref{Eqn:26}) that $\frac{d \hat{v}}{dt} = C\hat{v}(1-\varepsilon |\bm{\xi}|^2)|\bm{\xi}| = 0$. Hence, the perturbation does not grow, and the solution remains stable. Thus if $\varepsilon $ is larger than a critical value, we have $(1-\varepsilon |\bm{\xi}|^2) < 0$ the perturbations with all $|\bm{\xi}|$ decay, i.e, $|\hat{v}| = |e^{C(1-\varepsilon |\bm{\xi}|^2)|\bm{\xi}|t}| \rightarrow 0$ as $t \rightarrow \infty$. 

Thus, in this case, the equilibrium solution for the Eqn.(\ref{Eqn objective D_s}), $P_{\phi} = P_{*} = C$ is stable, which can also indicate that the training process of $D_{\phi}$ to $\max \mathcal{L}_{D}$ is stable (if $\varepsilon $ is larger than a critical value).

\item \textbf{General $P_{*}$ case.}
The stability and instability are local effects. In this case where $P_{*}$ is not constant, $P_{*}$ can still be approximated as a constant locally, which allows the above analysis to be applied.

\end{itemize}
\end{proof}

\section{Implement details} \label{Appendix4}

We describe the network architectures in our EIEG GAN. For the 25-Gaussians example, we use multi-layer perceptron (MLP) networks for the generator and the elastic discriminator. For the image generation example,  we use convolutional neural network (CNN) architectures.

\paragraph{25-Gaussians Example.}
\begin{itemize}
    \item The MLP elastic discriminator in EIEG GAN takes a 2-dimensional tensor as the input. Its architecture has a set of fully-connected layers (fc marked with input-dimension and output-dimension) and LeakyReLU layers (hyperparameter set as 0.2): \textsl{fc (2 $\rightarrow$ 100), LeakyReLU, fc (100 $\rightarrow$ 50), LeakyReLU, fc (50 $\rightarrow$ 2)}.
    \item  The MLP generator network in EIEG GAN takes a 2-dimensional random Gaussian variables as the input. Its architecture: \textsl{fc (2 $\rightarrow$ 100), LeakyReLU, fc (100 $\rightarrow$ 50), LeakyReLU, fc (50 $\rightarrow$ 2)}.
\end{itemize}

\paragraph{Image generation.}
\begin{itemize}
    \item The CNN elastic discriminator in EIEG GAN takes a $B \times C \times H \times W$ tensor as the input. Its architecture has a set of convolution layers (conv marked with input-c, output-c, kernel-size, stride, padding), Batch Normalization layers (BN) and LeakyReLU layers (hyperparameter as 0.2): \textsl{conv (3,64,4,2,1), LeakyReLU, conv (64,128,4,2,1), BN, LeakyReLU, conv (128,256,4,2,1), BN, LeakyReLU, conv (256,512,4,2,1)}.
    \item  The CNN generator network in EIEG GAN given a $100$ dimensional random Gaussian variables: \textsl{conv (100,256,4,2,0), BN, ReLU, conv (256,128,4,2,1), BN, ReLU, conv (128,64,4,2,1), BN, ReLU, conv (64,32,4,2,1), Tanh.}.
\end{itemize}

\paragraph{Hyperparameters}
The hyperparameters are given in Table\ref{Table: Hyperpara_EIEGGAN}.

\begin{table}[!hbtp]
\begin{center}
\begin{tabular}{cc}
\hline
Hyperparameters& settings\\
\hline
Cut-off $R_1$  & 0.1\\
Cut-off $R_2$ (stabilizing) & 0.8 \\
Stabilizing coefficient $\varepsilon$ & 1 \\
Learning rate $Lr_D$ & 1e-5\\
Learning rate $Lr_G$ & 1e-4\\
Discriminator iteration $n_c$ & 3\\
Batch size $B$ & 64\\
Input dim & 64\\
Generative size $H \times W$ & $32 \times 32$\\
\hline
\end{tabular}
\end{center}
\caption{Hyperparameters for training EIEG GAN.}
\label{Table: Hyperpara_EIEGGAN}
\end{table}

\section{Analysis of experimental results in the main paper} \label{section5}

In this section, we present more detailed analysis of the experiments presented in the paper.

\subsection{KDE plots for 25-Gaussians Example.} \label{section5_1}
We have presented the data sample points of the 25-Gaussians Example in Figure 8 in the main paper. Here we visualize the generated points by kernel density estimation (KDE) in Fig.\ref{Fig: 55KDE}. As illustrated in Fig.\ref{Fig: 55KDE}, we can see that compared with Vanilla GAN and WGAN GP, our EIEG GAN has better performance.

\begin{figure}[!hbtp]
    \centering
    \begin{subfigure}% FashionMNIST
    \centering
        \includegraphics[width=0.4\textwidth]{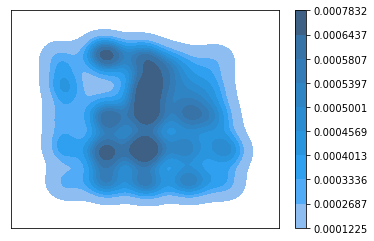}
    \end{subfigure}   
    \begin{subfigure} %[CIFAR10]
    \centering
        \includegraphics[width=0.4\textwidth]{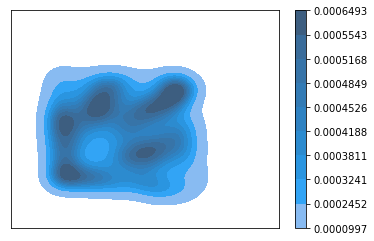}
    \end{subfigure}
    \begin{subfigure} %[CIFAR10]
    \centering
        \includegraphics[width=0.4\textwidth]{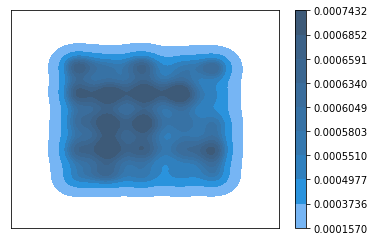}
    \end{subfigure}
    \caption{KDE plots of the generative samples, the first row is generated by Vanilla GAN \cite{ref1}, the second row is generated by WGAN GP \cite{ref2} and the last row is generated by our EIEG GAN.}
    \label{Fig: 55KDE}
\end{figure}

\subsection{Higher dimension of feature space may lead to better generative results.} \label{Appendix5_2}
As shown in the main paper, Table 1 in section 3, in our EIEG GAN higher dimension of feature space may lead to higher quality of generative samples. We believe this is because higher-dimensional feature spaces can capture more complex and subtle patterns in the data, which can lead to more realistic and varied generated samples. Here, we present more details for this property, including the generated samples and the corresponding t-SNE visualization \cite{refSNE}.
 
\begin{figure}[!hbtp]
    \centering
    \begin{subfigure}% FashionMNIST
    \centering
        \includegraphics[width=0.23\textwidth]{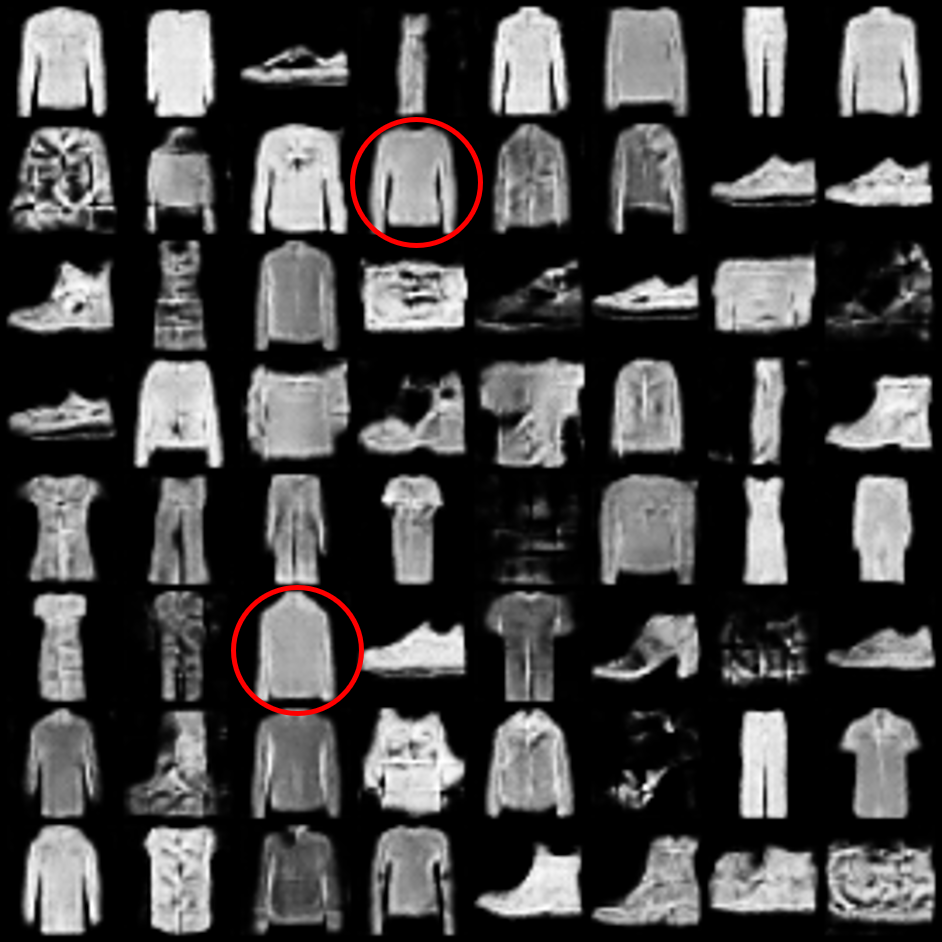}
    \end{subfigure}   
    \begin{subfigure} %[CIFAR10]
    \centering
        \includegraphics[width=0.23\textwidth]{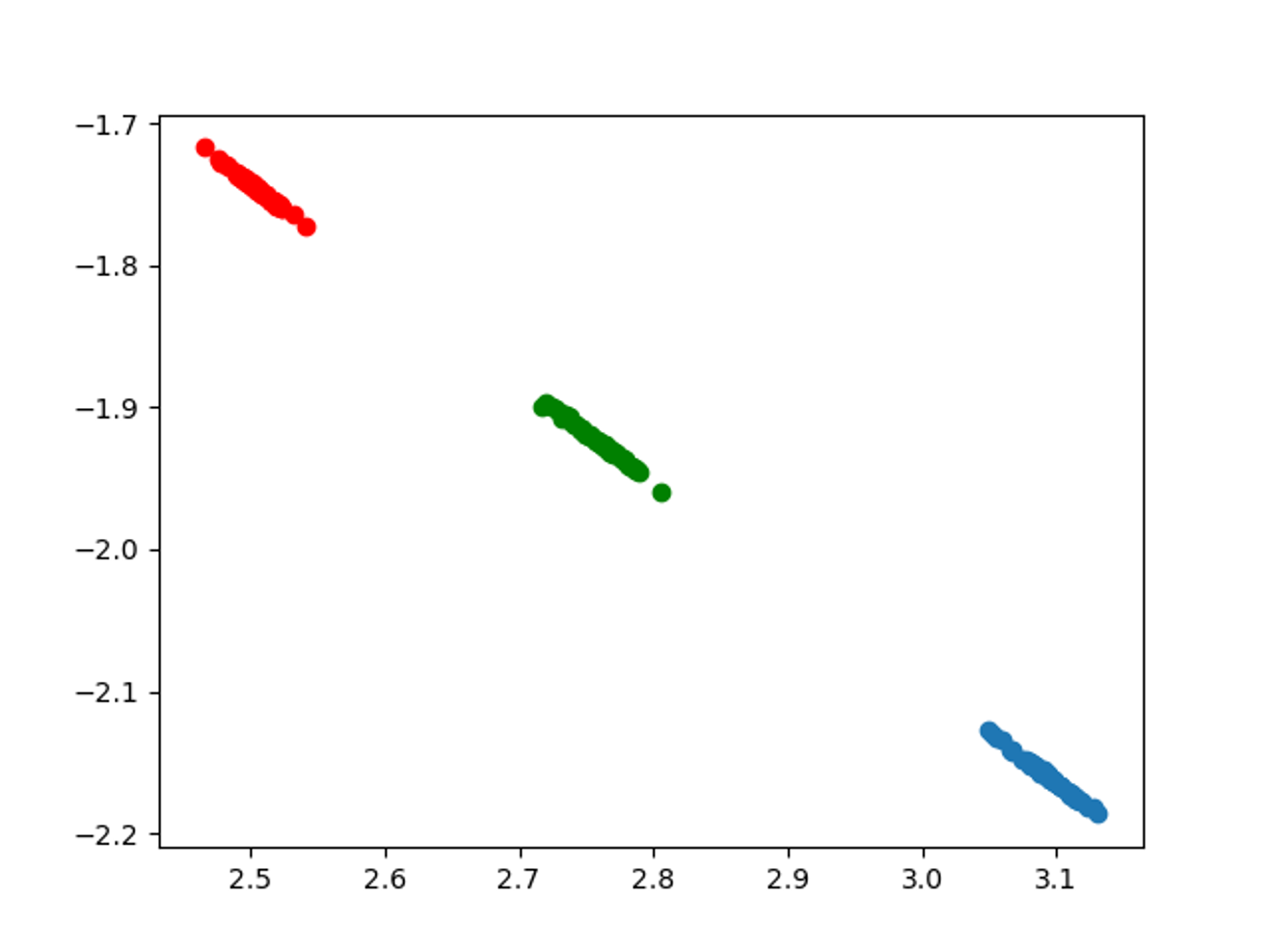}
    \end{subfigure}
    \begin{subfigure} %[CIFAR10]
    \centering
        \includegraphics[width=0.23\textwidth]{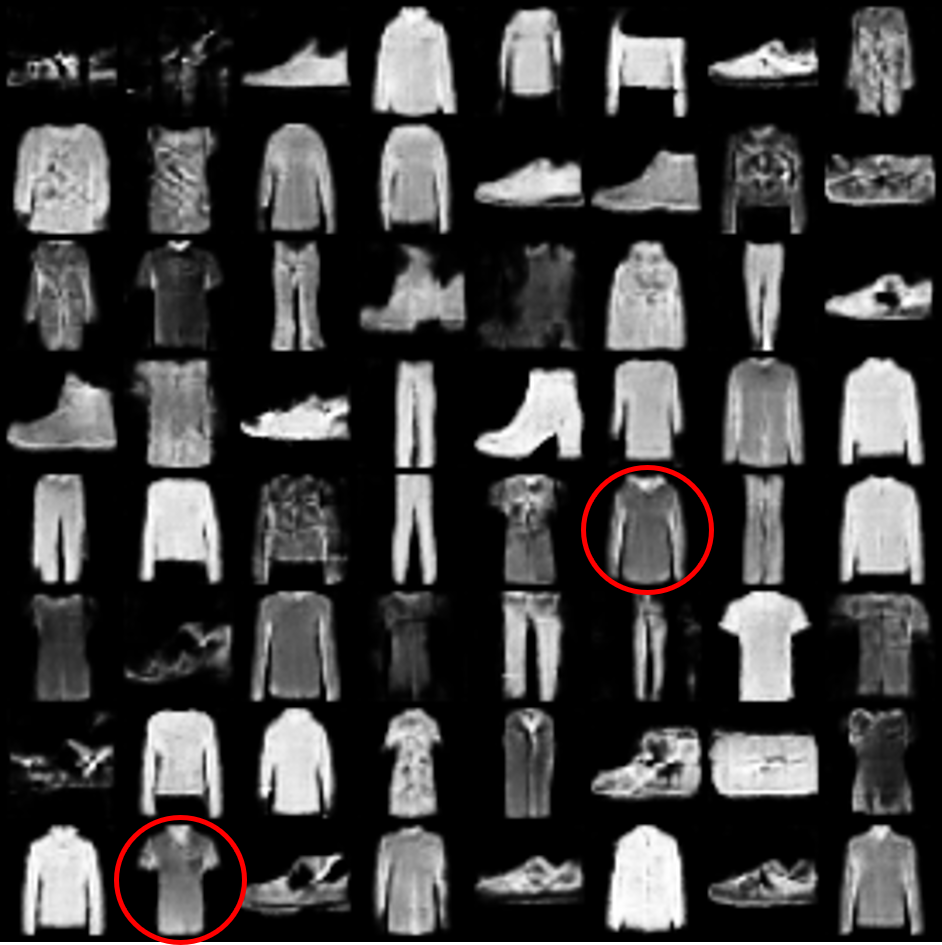}
    \end{subfigure}
    \begin{subfigure} %[CIFAR10]
    \centering
        \includegraphics[width=0.23\textwidth]{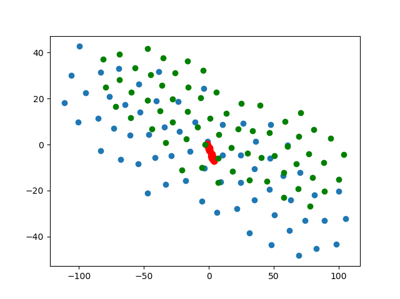}
    \end{subfigure}
    \begin{subfigure} %[CIFAR10]
    \centering
        \includegraphics[width=0.23\textwidth]{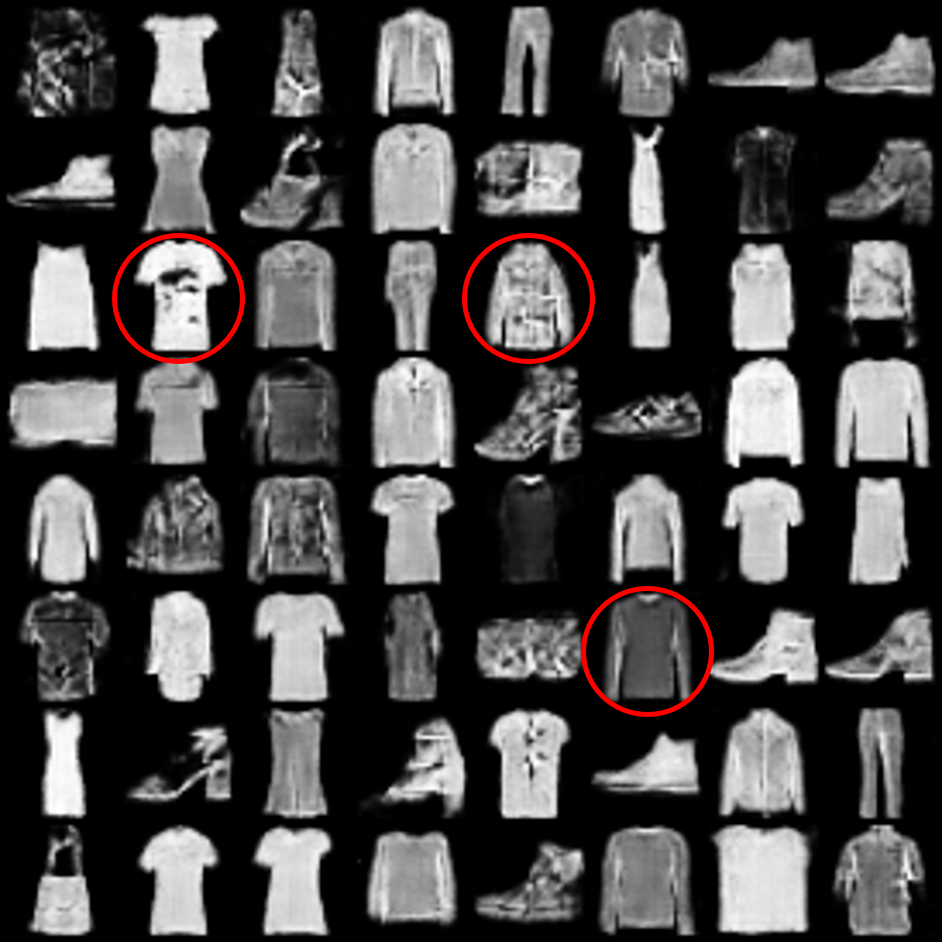}
    \end{subfigure}
    \begin{subfigure} %[CIFAR10]
    \centering
        \includegraphics[width=0.23\textwidth]{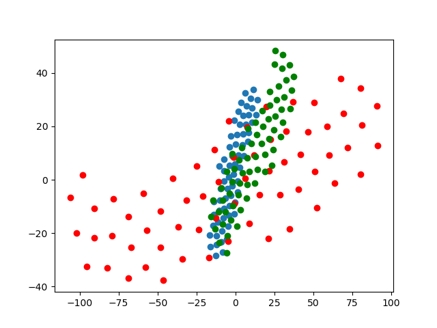}
    \end{subfigure}
    \caption{Generated samples on FashionMNIST and its corresponding t-SNE visualization. From top to bottom, the dimension of the feature space is $2$, $16$, and $32$. Right column: t-SNE visualization of the feature map. Blue points: feature points of the real samples. Red points: feature points of the generative samples before training $G_{\theta}$. Green points: feature points of the after training $G_{\theta}$. }
    \label{Fig: feature_size}
\end{figure}  

As presented in Fig.\ref{Fig: feature_size}, the level of detail preserved in the generative samples increases with the dimension of the feature space. When the dimension of the feature space is $2$ as shown in the first row of Fig.\ref{Fig: feature_size}, the generated samples may only capture the general outline of the clothing items in the original dataset. When the dimension of the feature space is increased to $16$ (second row of Fig.\ref{Fig: feature_size}) or $32$ (third row of Fig.\ref{Fig: feature_size}), the generative samples depict additional details such as collars and different patterns of clothing items. These observations confirm our intuition that an increased feature space dimensionality better captures the intricacies of the original dataset, resulting in higher quality generated images. We note that each dataset may have its own optimal feature space dimensionality.

In addition, the Inception Score \cite{ref21} increases as the dimensionality of the feature space increases as demonstrated in Fig.\ref{Fig: Learining curvefeature_size}.

\begin{figure}[!hbtp]
    \centering
    \begin{subfigure}% FashionMNIST
    \centering
        \includegraphics[width=0.4\textwidth]{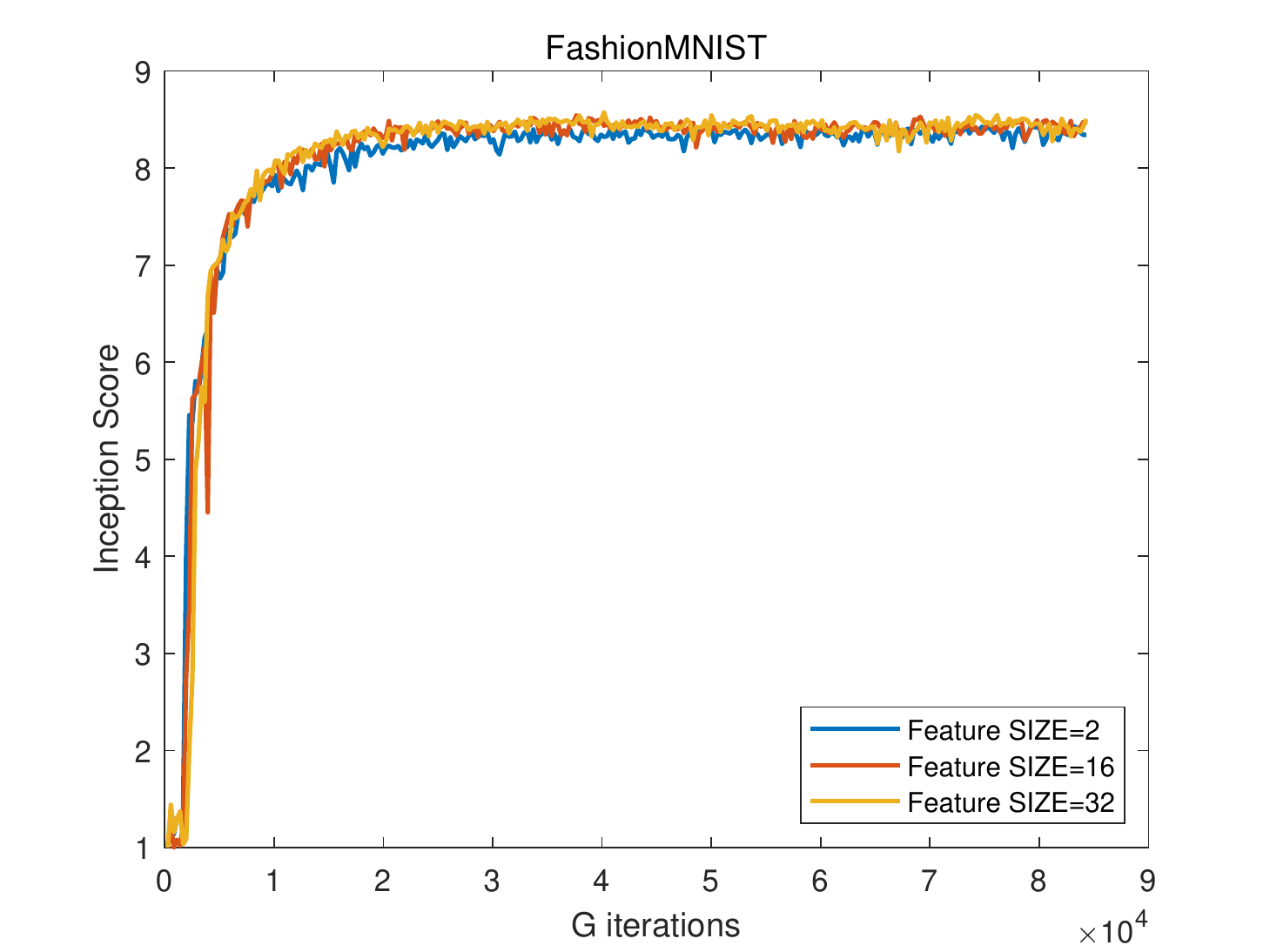}
    \end{subfigure}   
    \caption{Traning Curves on FashionMNIST with different dimensions of the feature space.}
    \label{Fig: Learining curvefeature_size}
\end{figure}

\end{appendices}
\end{document}